\documentclass[journal]{IEEEtran}
\usepackage{amsmath,graphicx}
\usepackage{cite}
\usepackage{algorithm}
\usepackage{algpseudocode}
\usepackage{graphicx}
\usepackage{textcomp}
\usepackage{amsmath,amsfonts, amssymb, graphicx,amsthm,epsfig}
\usepackage{bm,bbm} 
\usepackage{dsfont}
\usepackage{mathtools}
\usepackage{array,multirow}
\usepackage{tikz-network}
\usepackage{subfig}
\usepackage{float}
\usepackage{caption}
\usepackage{tikz}
\usepackage{upgreek}

\newcommand{\bxi}{\bm{\xi}}

\newcommand{\bpsi}{\bm{\psi}}
\newcommand{\bmu}{\bm{\mu}}
\newcommand{\bS}{\bm{\lambda}}

\newcommand{\e}{\mathbb{E}}

\newcommand{\brma}{\bm{a}}

\newcommand{\bq}{\bm{q}}

\newcommand{\blambda}{\bm{\lambda}}
\newcommand{\bLambda}{\bm{\Lambda}}
\newcommand{\wblambda}{\widetilde{\bm{\lambda}}}
\newcommand{\wlambda}{\widetilde{\lambda}}

\newcommand{\bx}{\bm{x}}

\newcommand{\E}{\mathbb{E}}

\newcommand{\dkl}{D_{\textup{KL}}}

\DeclareMathOperator{\T}{\mathsf{T}}
\DeclareMathOperator{\diag}{diag}

\renewcommand{\qedsymbol}{$\blacksquare$}

\newtheorem{theorem}{Theorem}

\newtheorem{corollary}{Corollary}
\newtheorem{assumption}{Assumption}

\newtheorem{application}{Application}

\def \rma{\mathrm{a}}
\def \brma{\bm{\mathrm{a}}}
\def \brmx{\bm{\mathrm{x}}}
\def \rmd{\mathrm{d}}
\def \rms{\mathrm{s}}
\def \rmp{\mathrm{p}}
\def \rmP{\mathrm{P}}
\def \rmA{\mathrm{A}}
\def \rmR{\mathrm{R}}
\def \rmU{\mathrm{U}}
\def \rmI{\mathrm{I}}
\def \rmM{\mathrm{M}}

\algtext*{EndWhile}

\begin{document}

\title{Causal Influence in Federated Edge Inference}

\author{Mert Kayaalp,
        Yunus \.Inan,
        Visa Koivunen,
        Ali H. Sayed   
\thanks{M. Kayaalp, Y. \.Inan, and A. H. Sayed are with the École Polytechnique Fédérale de Lausanne (EPFL), Switzerland. V. Koivunen is with the Aalto University, Finland. Corresponding author: M. Kayaalp. Email: mert.kayaalp@epfl.ch} 
}
\maketitle

\begin{abstract}
In this paper, we consider a setting where heterogeneous agents with connectivity are performing inference using unlabeled streaming data. Observed data are only partially informative about the target variable of interest. In order to overcome the uncertainty, agents cooperate with each other by exchanging their local inferences with and through a fusion center. To evaluate how each agent influences the overall decision, we adopt a causal framework in order to distinguish the actual influence of agents from mere correlations within the decision-making process. Various scenarios reflecting different agent participation patterns and fusion center policies are investigated. We derive expressions to quantify the causal impact of each agent on the joint decision, which could be beneficial for anticipating and addressing atypical scenarios, such as adversarial attacks or system malfunctions. We validate our theoretical results with numerical simulations and a real-world application of multi-camera crowd counting.
\end{abstract}

\section{Introduction}

\IEEEPARstart{A}{utonomous} systems are generally equipped with sensing and computing capabilities and wireless connectivity to enable prompt decisions based on real-time streaming observations. One example is self-driving vehicles, which need to react to changes in road and other traffic conditions in real-time --- see Fig.~\ref{fig:federated_cars} for a visual illustration. In situations involving multiple agents observing a common state of nature or phenomenon, cooperative decision-making becomes beneficial because the individual sensor observations may only carry partial information about the phenomenon of interest \cite{Sayed14}. For instance, the variable of interest might not be directly observable and might require an estimation using observable signals. Alternatively, agents might only see parts of the phenomenon or experience partial observability due to obstructions or interference.

The main advantage of cooperative decision-making lies in the diverse information provided by distinct agents. However, this very strength introduces its own challenges. On one hand, reliance on potentially outlying observations can lead to erroneous inferences or expose the system to adversarial threats. On the other hand, some outliers may provide critical observations that are crucial for informed decisions. Therefore, understanding to what extent an agent impacts the decision in a multi-agent system is an important question for many applications. 

In this study, we approach the concept of influence in multi-agent systems from a \emph{causal} perspective \cite{pearl2009causality,peters2017elements,kayaalp2023causal}. This distinction is important because correlation-inducing confounding factors between agents can lead to wrong conclusions about the impact \cite{kayaalp2023causal}. For instance, sensors located near each other may gather similar data, which could falsely be interpreted as one sensor influencing the other. By employing a causal framework, we can more accurately distinguish the true influences within a multi-agent system. Two typical applications are the following:

 \begin{figure}[]
     \centering
     \includegraphics[width=0.98\columnwidth]{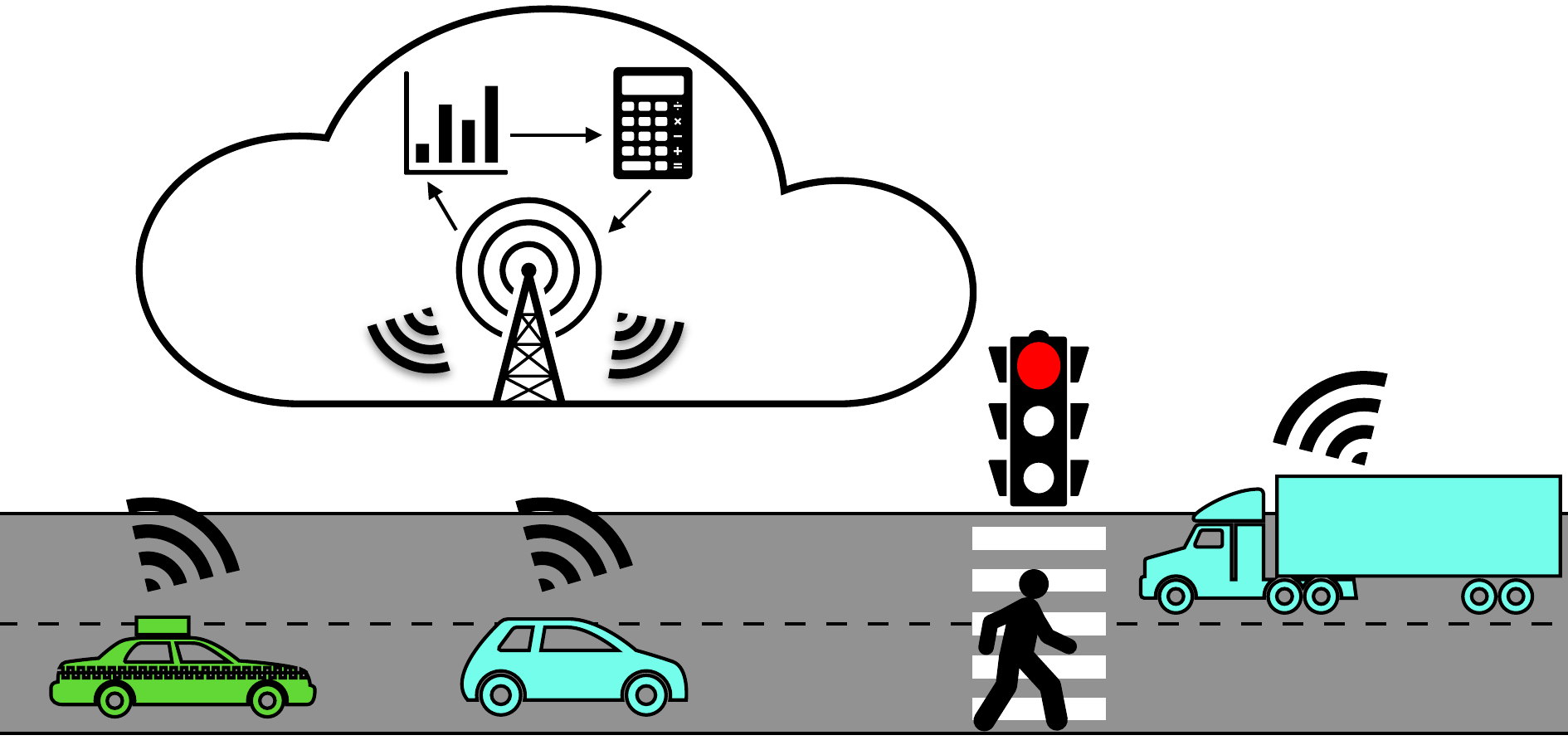}
     \caption[]{\small Intelligent vehicles and infrastructure can collaborate to enhance awareness of road conditions. Real-time and spontaneous cooperation is crucial in this context, as it allows for immediate responses to dynamic conditions, and hence improving the safety and efficiency of transportation.}
     \label{fig:federated_cars}
 \end{figure}

\begin{application}[\textbf{Vehicular ad-hoc networks}]\label{application:vanets}
Intelligent vehicles on the same road can collaborate to better understand the road and traffic conditions. These interactions typically happen in ad-hoc scenarios characterized by asynchronicity and potential latent confounding factors. In this context, to ensure robust decisions that appropriately consider outlier agents, our method can be applied to analyze each agent’s causal contribution.
\end{application}

\begin{application}[\textbf{Environmental sensing}] 
Cooperation among sensors is widely-used for accurately inferring the state of the environment or electromagnetic spectrum. Consider a scenario where one sensor starts transmitting data that significantly differs from the rest, perhaps due to some unexpected malfunction, interference or adversarial intent. Employing our methods to quantify the impact of such deviations on the joint decision can allow the fusion center to set thresholds for identifying and potentially discarding outlier data. \hfill \qedsymbol
\end{application}

\subsection{Contributions}

\begin{itemize}
\item  We build upon the collaborative decision-making framework of \cite{kayaalp2023fusion}, which involves heterogeneous agents exchanging beliefs (or soft-decisions) through a fusion center (FC) based on streaming observations. Furthermore, to better capture the real-world conditions, we will incorporate two asynchronicity scenarios to this framework; the scenarios differ in agent participation patterns and in FC policies.
\item By applying hypothetical interventions \cite{pearl2009causality,peters2017elements, kayaalp2023causal} on our model, we implement a method to calculate causal impact scores for each agent's contribution to the joint decision. We also provide a theoretical analysis of participation patterns, FC policies, and data distribution on the decision-making process.
\item We validate our theoretical findings with numerical simulations and also apply our methods to real-world data from a multi-camera crowd-size estimation application \cite{chavdarova_wildtrack}.
\end{itemize}

This paper is organized as follows. Section~\ref{sec:related_work} surveys the related work in the literature. Section~\ref{sec:problem_formulation} revisits the cooperative inference framework from \cite{kayaalp2023fusion}, and extends it with two distinct asynchronous agent behavior scenarios. Section~\ref{sec:causal_def} presents our definitions of the causal impact in these frameworks. Section~\ref{sec:theoretical_results} derives closed-form expressions for the impacts of agents and other theoretical contributions. Then, in Section~\ref{sec:numerical}, we illustrate our theoretical results using synthetic data and also apply our methods to a real-world scenario of crowd counting with multi-sensor data.

\section{Related Work}\label{sec:related_work}

\subsection{Causal influence estimation}

Influence analysis within multi-agent systems has roots in cooperative game theory, notably through the Shapley value that assesses an agent's marginal contribution to the collective output \cite{shapley_53}. There are also methods that quantify influence through centrality or application-specific metrics \cite{dablander2019node,valentina2023discovering}. However, these approaches can fall short in distinguishing causal influences from correlational associations. Therefore, in this work, we approach influence analysis from a structural causal modeling framework \cite{pearl2009causality,peters2017elements,kayaalp2023causal}, where we measure the effect of a factor on an outcome by \textit{intervening} on that factor while keeping other variables fixed.

Randomized controlled trials to understand interventional effects are often not feasible in real-world applications. In these cases, causal inference methods \cite{pearl2009causality} can help discover interventional effects using \textit{observational} data and a model. In fields like signal processing, control, and communications, the modular structure of systems usually offers a straightforward understanding of the data generative process. Unfortunately, this understanding is often neglected, although it can aid in the direct application of causal inference method. This is in contrast to the areas such as healthcare where the need to learn causal representations from data \cite{peters2017elements,proc_ieee_causality} remains a significant challenge.

\subsection{Cooperative inference}

Our work builds upon the framework of \cite{kayaalp2023fusion}, which is a special case of locally Bayesian social learning over networks \cite{jadbabaie_2012, zhao_2012, nedic_2017, lalitha_2018, bordignon2021adaptive, shaska_2023, kayaalp2022aaga_journal}. In social learning, agents cooperate by processing streaming observations and exchanging inference results among each other, instead of directly sharing data. As opposed to the related literature on decentralized detection and data fusion \cite{hoballah_89, tsitsiklis1988decentralized, tsitsiklis93, varshney2012book, zou2010cooperative, bajovic2012large, li2020_aa_tsp,inan2022fundamental}, in social learning, the emphasis is on the active computation and information exchange by intelligent agents rather than the passive information relay to a FC by the nodes. Moreover, the federated architecture we consider coincides with that of the federated learning literature that promotes cooperative model training across multiple entities \cite{mcmahan2017,li2020,kairouz2021advances,rizk2022federated,ozkara2023a,elbir2022federated} without data exchanges. However, unlike federated learning's focus on model training with distributed data, our work focuses on the collaborative inference in the post-training (prediction) phase. \\

\noindent \textbf{Notation:} Random variables are written in boldface letters. We use the ``proportional to'' symbol \(\propto\) whenever the LHS of an equation is a proper normalization of the RHS. For example, for \( \theta \in \Theta\) and a function \( f \):
\begin{equation}
    \mu (\theta) \propto f (\theta) \Longleftrightarrow \mu (\theta) = \frac{f (\theta)}{\sum_{\theta^\prime \in \Theta }f (\theta^\prime)}.
\end{equation}
The Kullback-Leibler (KL) divergence \cite{csiszar2011information} between two probability distributions \(p\) and \(q\) is denoted by \(\dkl (p||q)\). Moreover, $\mathds{1}_K$ denotes a vector of dimension-$K$ with all entries equal to $1$. Following the notation in \cite{kayaalp2023causal}, we use \(\sim\) to denote the counterparts of variables after an intervention (e.g., $\widetilde{\lambda}$ represents the variable $\lambda$ after an intervention).

\section{Federated Inference}\label{sec:problem_formulation}

\subsection{Synchronous Collaboration}\label{sec:fed_inference}

We start by revisiting the synchronous federated decision-making setting of \cite{kayaalp2023fusion} --- see Fig.~\ref{fig:federated_no_intervention}. Consider a setting where a group of \( K \) agents (e.g., clients, sensors, machines) wish to discover the true state of nature \( \theta^\circ \) from a set of $H$ potential hypotheses \( \Theta \triangleq \{\theta_1, \dots, \theta_H\} \), with the help of a fusion center (e.g., cloud, base station). For instance, autonomous vehicles on the same road can be connected to a cloud with the objective of assessing the road and traffic conditions (e.g., \{\textit{crowded}, \textit{accident}, \textit{normal}\}). 

At each time instant \( i \), each agent \( k \) acquires an observation \( \bxi_{k,i} \). This observation conveys partial information about $\theta^\circ$ due to each agent's potentially limited or noisy view of the overall phenomenon. Furthermore, data across agents are not necessarily assumed to be independent. This is common in applications where ensuring spatial independence of observations is impractical. In such environments, there might be confounding factors affecting multiple agents simultaneously, which makes interpreting influence from a causal perspective crucial.

Instead of directly transmitting the raw observations \( \bxi_{k,i} \) to the central server, each agent \( k \) processes its data locally with a personalized likelihood model \cite{sayed_2022}. This model serves as an approximation of the true data-generative process and can be learned, for example, using a neural network. Agent $k$ then incorporates the likelihood score $L_k(\bxi_{k,i} | \theta)$ into the Bayes' rule for obtaining an intermediate belief about which hypothesis is the true one:
\begin{align}\label{eq:dif_adapt_step}
 \bpsi_{k,i} (\theta) &\propto L_k(\bxi_{k,i} | \theta)\bmu_{i-1} (\theta) \quad \text{(Adapt)}
\end{align}
Here, $\bmu_{i-1}$ is the prior belief, which is a probability mass function (pmf) over $\Theta$. The symbol \(\propto\) is a shorthand notation for the following normalization:
\begin{align}\label{eq:dif_adapt_step_unnormal}
 \bpsi_{k,i} (\theta) = \dfrac{L_k(\bxi_{k,i} | \theta)\bmu_{i-1} (\theta)}{\sum_{\theta^\prime} L_k(\bxi_{k,i} | \theta^\prime)\bmu_{i-1} (\theta^\prime)}.
\end{align}
After the self-adaptation step \eqref{eq:dif_adapt_step}, agent \( k \) forwards the intermediate belief \( \bpsi_{k,i} \) to the fusion center (FC). The FC may lack knowledge about the system's joint data distribution, the observations at the agents, or the agents' likelihood models. This can be because of the spontaneous formation of the collaboration, or constraints on privacy. Therefore, the FC employs a weighted geometric averaging of the received information in a non-Bayesian manner \cite{jadbabaie_2012,zhao_2012,nedic_2017,lalitha_2018, bordignon2021adaptive}:
\begin{align}\label{eq:geometric_fusion}
   \bmu_{i}(\theta) &\propto  \prod_{k=1}^K (\bpsi_{k,i} (\theta ))^{\pi_k} \quad \text{(Combine)}.
\end{align}
Here, \( \pi \triangleq [\pi_1, \dots, \pi_K]^{\T} \) is a vector of confidence weights \( \pi_k \in (0,1) \) assigned by the fusion center to each agent \( k \) \cite{Sayed14,varshney2012book}, potentially formed from the previous interactions with the agents. These weights are assumed to be positive constants that add up to $1$. The server then sends the aggregated belief back to the agents. This procedure of local updating and exchanging of beliefs is executed repeatedly at every time instant. The procedure is summarized in Algorithm~\ref{alg:sync_alg}.

  \begin{figure}[]
     \centering
     \includegraphics[width=0.8\columnwidth]{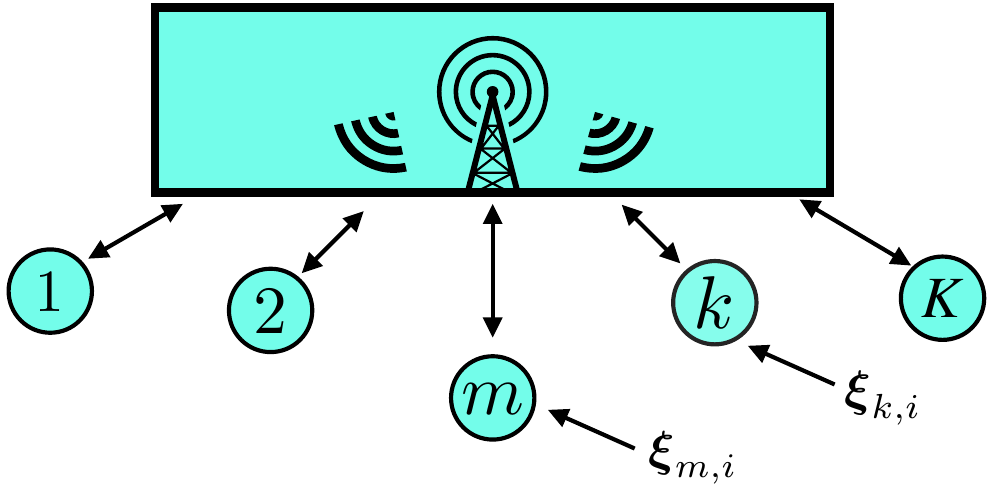}
     \caption[]{\small Visual representation of the federated inference framework. At each time instant $i$, $(a)$ each agent receives an external observation, $(b)$ processes it locally and transmits it to a fusion center (FC), and $(c)$ FC center broadcasts the combined soft-decision (belief) back to agents.}
     \label{fig:federated_no_intervention}
 \end{figure}

  \begin{algorithm}[]
 \caption{Synchronous federated inference}
 \begin{algorithmic}[1]
    \item set initial prior to $\bmu_{0}(\theta)>0$,  $\forall \theta \in\Theta$ and $\forall k$
    \While{$i\geq 1$} 
 \For{each agent $k$} 
 \State receive private observation $\bxi_{k,i}$; \State \textbf{adapt} to obtain intermediate belief:
\begin{equation}
\bpsi_{k,i} (\theta) \propto L_k(\bxi_{k,i} | \theta)\bmu_{i-1} (\theta) 
\end{equation}
\EndFor
\State \textit{all} agents send their local intermediate beliefs to FC;
\State FC \textbf{combines} the local beliefs:
\begin{align}
   \bmu_{i}(\theta) &\propto  \prod_{k=1}^K (\bpsi_{k,i} (\theta ))^{\pi_k} 
\end{align}
\State FC broadcasts $\bmu_{i}$ to \textit{all} agents;
\State \( i \leftarrow i+1\)
\EndWhile
 \end{algorithmic} \label{alg:sync_alg}
 \end{algorithm}

\subsection{Two Asynchronous Scenarios}\label{sec:async_behavior}

Asynchronous behavior is common in many real-world distributed systems, and is particularly relevant for ad-hoc networks where time scheduling beforehand may not be plausible. To that end, we consider two scenarios that are distinct based on the symmetry of communication between the agents and the FC. For both scenarios, we use the Bernoulli variable $\bq_{k,i}$ (with parameter $p_k$)
to indicate if agent $k$ is sharing its intermediate belief $\bpsi_{k,i}$ with the server at time $i$, namely,
\begin{equation}\label{eq:qki_pk_def}
    \bq_{k,i} =
    \begin{cases}
  1, \quad \text{with probability} \: p_k\\0, \quad \text{otherwise}
\end{cases}.
\end{equation}
We assume the process $\{\bq_{k,i}\}$ is independent and identically distributed (i.i.d.) over time and independent over space. 
 \subsubsection{Asymmetric communication} There can be instances when agents, despite being active, do not transmit information to the FC and remain idle in terms of data sharing. This non-engagement can be due to various factors, such as the need to conserve energy, particularly important for battery-operated agents where excessive transmission can lead to rapid battery depletion. Other reasons might include non-informative soft decisions, or the lack of significant changes in intermediate statistics since the previous transmission. However, these agents can keep receiving updates from the server. A possible reason for this disparity is that the uplink cost (from agent to server) is typically higher than the downlink cost (from server to agent). For instance, the downlink could be broadcast, i.e., the same message is transmitted to all agents without the need to exchange information separately with each individual agent. In this case, the FC can fill the belief components of missing agents with its own prior while aggregating information. Therefore, the combination step \eqref{eq:geometric_fusion} at the server side changes to  
\begin{align}\label{eq:geometric_fusion_c1}
   \bmu_{i}(\theta) &\propto  \prod_{k=1}^K \Big (\bpsi_{k,i}^{\bq_{k,i}} (\theta )  \bmu_{i-1}^{1-\bq_{k,i}} (\theta) \Big )^{\pi_k}.
\end{align}
Nevertheless, the adaptation step \eqref{eq:dif_adapt_step} at the agent side remains unchanged and agents continue to utilize the beliefs received from the server locally. The procedure under asymmetric asynchronicity is summarized in Algorithm~\ref{alg:asymmetric}.

It is worth noting the parallel between this scenario and the traditional distributed detection strategies \cite{varshney2012book,tsitsiklis1988decentralized,inan2022fundamental}. Since the server knows the previous combined belief $\bmu_{i-1}$, the action of sharing intermediate beliefs $\{\bpsi_{k,i}\}$ by agents is essentially equivalent to them sharing the observation likelihoods $\{L_k(\bxi_{k,i} | \theta)\}$ due to \eqref{eq:dif_adapt_step}. Similarly, nodes (e.g., sensors) relay a sufficient statistics such as their local likelihoods or likelihood ratios to the FC in \cite{varshney2012book,tsitsiklis1988decentralized,inan2022fundamental}. The difference is that in those works, the FC does not communicate any information back to the nodes. \\

  \begin{algorithm}[]
 \caption{Asymmetric communication}
 \begin{algorithmic}[1]
    \item set initial prior to $\bmu_{0}(\theta)>0$,  $\forall \theta \in\Theta$ and $\forall k$
    \While{$i\geq 1$} 
 \For{each agent $k$} 
 \State receive private observation $\bxi_{k,i}$; \State \textbf{adapt} to obtain intermediate belief:
\begin{equation}
\bpsi_{k,i} (\theta) \propto L_k(\bxi_{k,i} | \theta)\bmu_{i-1} (\theta) 
\end{equation}
\EndFor
\State each agent $k$ will send its intermediate belief to FC \textit{if} $\bq_{k,i} = 1$ (with probability $p_k$);
\State FC \textbf{combines} the received beliefs and its own prior: 
\begin{align}
   \bmu_{i}(\theta) \propto  \prod_{k=1}^K \Big (\bpsi_{k,i}^{\bq_{k,i}} (\theta )  \bmu_{i-1}^{1-\bq_{k,i}} (\theta) \Big )^{\pi_k}
\end{align}
\State FC broadcasts $\bmu_{i}$ to \textit{all} agents
\State \( i \leftarrow i+1\)
\EndWhile
 \end{algorithmic} \label{alg:asymmetric}
 \end{algorithm}

\subsubsection{Symmetric communication} Another possibility is that an agent does not receive any update from the server if that agent has not transmitted information to the central processor at that time instant. In other words, the absence of communication is reciprocal. This situation could arise in cases where the quality of the connection is not adequate for reliable communication.
Alternatively, for various reasons, a server might strategically choose not to update agents that do not contribute information. By doing so, it can incentivize data sharing and promote a give-and-take dynamics. In this scenario, the combination step at the server side is given by \eqref{eq:geometric_fusion_c1}, whereas the adaptation step \eqref{eq:dif_adapt_step} at the agents becomes
\begin{align}\label{eq:dif_adapt_step_c2}
 \bpsi_{k,i} (\theta) &\propto \begin{cases}
     L_k(\bxi_{k,i} | \theta)\bmu_{i-1} (\theta), \quad &\text{if} \: \bq_{k,i-1} = 1  \\
     L_k(\bxi_{k,i} | \theta)  \bpsi_{k,i-1} (\theta), \quad &\text{if} \: \bq_{k,i-1} = 0  
 \end{cases} .
\end{align}
The rationale behind \eqref{eq:dif_adapt_step_c2} is as follows. If agent $k$ has shared information with the server (i.e., $\bq_{k,i-1} = 1$) at time $i-1$, the server returns the combined belief $\bmu_{i-1}$ to that agent. On the other hand, if the agent has not participated in the information exchange (i.e., $\bq_{k,i-1} = 0$), then the server does not provide the updated belief and the agent resorts to its own belief $\bpsi_{k,i-1}$ as a prior for the update at the next time instant $i$. The procedure under symmetric asynchronicity is summarized in Algorithm~\ref{alg:symmetric}.

  \begin{algorithm}[]
 \caption{Symmetric communication}
 \begin{algorithmic}[1]
    \item set initial prior to $\bmu_{0}(\theta)>0$,  $\forall \theta \in\Theta$ and $\forall k$
    \While{$i\geq 1$} 
 \For{each agent $k$} 
 \State receive private observation $\bxi_{k,i}$ \State \textbf{adapt} to obtain intermediate belief:
\begin{align}
 \bpsi_{k,i} (\theta) &\propto \begin{cases}
     L_k(\bxi_{k,i} | \theta)\bmu_{i-1} (\theta), \quad &\text{if} \: \bq_{k,i-1} = 1  \\
     L_k(\bxi_{k,i} | \theta)  \bpsi_{k,i-1} (\theta), \quad &\text{if} \: \bq_{k,i-1} = 0  
 \end{cases}
\end{align}
\EndFor
\State each agent $k$ will send its intermediate belief to FC \textit{if} $\bq_{k,i} = 1$ (with probability $p_k$)
\State FC \textbf{combines} the received beliefs and its own prior: 
\begin{align}
   \bmu_{i}(\theta) \propto  \prod_{k=1}^K \Big (\bpsi_{k,i}^{\bq_{k,i}} (\theta )  \bmu_{i-1}^{1-\bq_{k,i}} (\theta) \Big )^{\pi_k}
\end{align}
\State FC sends $\bmu_{i}$ \textit{only} to agents that have participated in the cooperation in the current round (i.e., $\bq_{k,i} = 1$)
\State \( i \leftarrow i+1\)
\EndWhile
 \end{algorithmic} \label{alg:symmetric}
 \end{algorithm}

\section{Causal Influence}\label{sec:causal_def}

We extend the causal effect definition from \cite{kayaalp2023causal}. The main motivation for the definition is that the influence of an agent $m$ on the collective decision should be proportional to the ``amount'' by which the outcome changes when this agent is intervened upon. In other words, the alteration of the outcome under a manipulation on the agent quantifies the causal influence. To this end, when an intervention occurs on agent $m$, we decouple its belief $\bpsi_{m,i}$ from other beliefs and observations and \emph{fix} it at some constant pmf, say, $\bpsi_{m,i}=\mu_m$ --- see Fig.~\ref{fig:federated_intervention} for a representation of an intervention on Fig.~\ref{fig:federated_no_intervention}. 

As an illustration, recall Application~\ref{application:vanets} on cooperative vehicular networks. Consider a scenario in which these vehicles navigate a dry road while receiving noisy data from their sensors. To measure the causal effect of an individual vehicle on the collective decision, we can ask the following question: How would the group's decision on road conditions change if a single vehicle, despite the actual conditions and data from other vehicles, consistently reported that the road is icy? A significant difference in the collective decision would imply an influential role for that vehicle. Conversely, a minimal alteration suggests a negligible causal effect. Furthermore, this effect is a \textit{causal} effect since the hypothetical intervention we consider directly targets the agent. Namely, it is irrespective of other environmental factors and vehicles that might otherwise induce non-causal correlations.

We will establish in Theorem~\ref{theorem:no_intervention} that in the absence of any intervention, the belief vector $\bmu_i$ converges to a steady-state value $\mu_{\infty}$ that places a probability value of $1$ on the true hypothesis $\theta^\circ$ as $i \to \infty$. When an intervention occurs at agent $m$, the steady-state belief vector will be denoted by $\widetilde{\mu}_{\infty}$. As such, we can quantify the causal impact of agent $m$ on the joint decision by using the difference:
\begin{equation}\label{eq:cm_general_def}
C_m \triangleq 1 - \widetilde{\mu}_{\infty}(\theta^{\circ}).
\end{equation}
Expression \eqref{eq:cm_general_def} measures the expected shift in the steady-state belief on the true hypothesis $\theta^\circ$ due to an intervention on agent $m$. Note that as in \cite{kayaalp2023causal}, we express the belief $\widetilde{\mu}_{\infty}(\theta^{\circ})$ in the form:
\begin{equation}\label{eq:expected_log_belief_trans}
\widetilde{\mu}_{\infty}(\theta^{\circ}) \triangleq \dfrac{1}{1 +\sum\limits_{\theta \neq \theta^\circ} \exp \{ - \widetilde{\lambda}_{\infty}(\theta) \}}.
\end{equation}
The variable $\widetilde{\lambda}_{\infty}(\theta)$ is defined as follows. First, we introduce the notation 
\begin{equation}
    \lambda_{\infty}(\theta) \triangleq \lim_{i \to \infty} \e [ \bS_{i}(\theta) ] 
\end{equation}
to represent the expected log-belief ratio in the limit with the variables $\bS_{i}(\theta)$ defined by
\begin{equation}
 \bS_{i}(\theta) \triangleq \log \dfrac{\bmu_{i} (\theta^\circ)}{\bmu_{i} (\theta)}.
\end{equation}
Then, we recall that we use \(\sim\) to denote interventional counterparts of these variables, which means
\begin{equation}
 \widetilde{\bS}_{i}(\theta) \triangleq \log \dfrac{\widetilde{\bmu}_{i} (\theta^\circ)}{\widetilde{\bmu}_{i} (\theta)}
\end{equation}
represents the log-belief ratio under an intervention. Therefore, the variable
\begin{equation}
    \widetilde{\lambda}_{\infty}(\theta) \triangleq \lim_{i \to \infty} \e [ \widetilde{\bS}_{i}(\theta) ]
\end{equation}
represents the expected asymptotic log-belief ratio under an intervention. Note that the existence of the limits are guaranteed under the mild assumptions of finite KL divergences and positive initial beliefs \cite{kayaalp2023causal}, which will be introduced in the next section.

  \begin{figure}[]
     \centering
     \includegraphics[width=0.8\columnwidth]{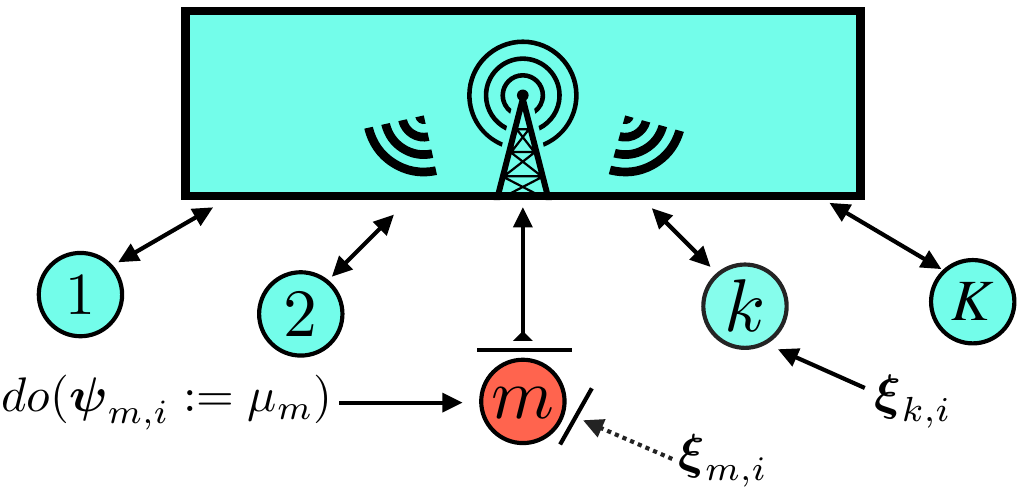}
     \caption[]{\small Visual representation of a hypothetical intervention $do(\bpsi_{m,i} := \mu_m)$ on the graphical model in Fig.~\ref{fig:federated_intervention}. Agent $m$ keeps sending information to the server with probability $p_m$, however, its belief is now fixed and is not dependent on any other variable.}
     \label{fig:federated_intervention}
 \end{figure}

\section{Theoretical Results}\label{sec:theoretical_results}

We begin by examining the system in the observational (i.e., pre-intervention) mode of the system. For this, we first define the informativeness level of each agent $k$ as
\begin{equation}
d_k (\theta) \triangleq \dkl \big (L_k(\cdot|\theta^{\circ}) || L_k(\cdot|\theta) \big)
\end{equation}
which represents how informative agent $k$'s observations are for distinguishing the true hypothesis $\theta^\circ$ from some other hypothesis $\theta$. We assume that the likelihood functions $L_k(\cdot|\theta)$ have the same support, which is a necessary condition to remove pathological cases in which a single observation can be sufficient to distinguish the truth with absolute certainty. More formally,
\begin{equation}
    \dkl( L_k (\cdot | \theta^\circ ) ||  L_k (\cdot | \theta )) < \infty 
\end{equation}
for each agent \( k \) and for each hypothesis \( \theta \in \Theta \). Furthermore, the following is a canonical assumption in local Bayesian learning \cite{jadbabaie_2012,zhao_2012, nedic_2017, lalitha_2018, bordignon2021adaptive} to ensure that aggregate of all agents can distinguish the true hypothesis from the wrong ones.
\begin{assumption}[\textbf{Global identifiability}]\label{assum:global_identification}
For each incorrect hypothesis \( \theta \in \Theta \setminus \{\theta^\circ\} \), there exists at least one agent \( k_\theta \) with \(  \dkl( L_{k_\theta} (\cdot | \theta^\circ ) ||  L_{k_\theta} (\cdot | \theta )) > 0 \) that can distinguish $\theta$ and $\theta^\circ$. \hfill\qedsymbol
\end{assumption}
Moreover, for the FC not to discard any hypothesis from the beginning, initial beliefs need to have full support, namely, that $\bmu_{0}(\theta)>0$  $\forall \theta \in\Theta$ \cite{kayaalp2022aaga_journal}. Under these conditions, the following result describes the evolution of the beliefs under no intervention.
\begin{theorem}[\textbf{Pre-intervention}]\label{theorem:no_intervention}
   For the synchronous as well as the symmetric and asymmetric asynchronous communication protocols discussed in Sec.~\ref{sec:problem_formulation}, the belief vector $\bmu_i$ converges to a steady-state probability mass function that places a value of $1$ on the true hypothesis $\theta^\circ$ almost surely:
   \begin{equation}
       \lim_{i \to \infty} \bmu_i (\theta^\circ) = 1 \quad \text{with probability 1}.
   \end{equation}
\end{theorem}
\begin{proof}
    Appendix~\ref{appendix:pre_int}.
\end{proof}
The causal influence of an agent $m$ on the joint decision is characterized by the shift of the overall decisions between pre and post-interventions. Therefore, we proceed to examine the beliefs under an intervention on agent $m$. We first review the causal impact result from \cite{kayaalp2023causal}, which addresses synchronous communication.

\begin{theorem}[\textbf{Synchronous collaboration \cite{kayaalp2023causal}}]\label{prop:sync_collab} Under synchronous collaboration described in Sec.~\ref{sec:fed_inference}, the expected log-belief ratio under intervention is given by
\begin{equation}\label{eq:sync_lambda}
  \widetilde{\lambda}_{\infty}(\theta) = \frac{1}{\pi_m} \sum_{k \neq m} \pi_{k} d_{k} (\theta) +  \log \frac{\mu_{m}(\theta^{\circ})}{\mu_{m}(\theta)}
\end{equation}
Therefore, by \eqref{eq:cm_general_def},  the causal impact of agent $m$ on the joint decision is 
\begin{align}\label{eq:sync_cm}
      C_m \!= \!1\!-\dfrac{1}{1 + \!\!\sum\limits_{\theta \neq \theta^\circ} \dfrac{\mu_{m}(\theta)}{\mu_{m}(\theta^\circ)}  \exp \Big \{ \!- \!\dfrac{1}{\pi_m} \!\sum\limits_{k \neq m} \pi_{k} d_{k} (\theta) \! \Big \}}
\end{align}
\end{theorem} \qed

\noindent Equations \eqref{eq:sync_lambda} and \eqref{eq:sync_cm} imply that:
\begin{itemize}
\item  An increase in the confidence $\pi_m$ by the fusion center increases the causal impact of agent $m$.
\item  Increasing the informativeness and confidence weights of the other agents decreases the impact of agent $m$.
\end{itemize}
Also, observe that \eqref{eq:sync_lambda} and \eqref{eq:sync_cm} are dependent on the intervention strength $\mu_m$. This is typical for interventions on continuous valued variables \cite{peters2017elements}. To have an intervention-dose independent causal impact measure, setting $\mu_m$ to a uniform belief is discussed in \cite{kayaalp2023causal}. Specifically, reference \cite{kayaalp2023causal} shows that setting the log-belief ratio $\log \frac{\mu_{m}(\theta^{\circ})}{\mu_{m}(\theta)}$ to $0$ is equivalent to the causal derivative effects discussed in \cite{peters2017elements}. Next, we consider the causal influences for the asynchronous scenarios we have introduced in Sec.~\ref{sec:async_behavior}.

\begin{theorem}[\textbf{Asymmetric communication}]\label{theorem:asymmetric}
    Under the asymmetric communication protocol described in Sec.~\ref{sec:async_behavior}, the expected log-belief ratio under intervention is given by
\begin{align}\label{eq:asymmetric_th_lambda}
\widetilde{\lambda}_{\infty}(\theta) &= \frac 1 {\pi_m} \sum_{k \neq m} \pi_{k} p_{k} d_{k} (\theta) +  p_m \log \frac{\mu_{m}(\theta^{\circ})}{\mu_{m}(\theta)} 
\end{align}
This implies by \eqref{eq:cm_general_def} that the causal effect of agent $m$ on the joint decision is given by
\begin{align}
      C_m \!=\! 1\!-\!\dfrac{1}{1 + \!\sum\limits_{\theta \neq \theta^\circ} \left(\dfrac{\mu_{m}(\theta)}{\mu_{m}(\theta^\circ)}\right)^{p_m}  \!\!\!\!\!\exp \Big \{ \!- \!\dfrac{1}{\pi_m} \sum\limits_{k \neq m} \pi_{k} p_k d_{k} (\theta) \! \Big \}}
\end{align}
\end{theorem}
\begin{proof}
    Appendix~\ref{appendix:asymmetric}.
\end{proof}
Notice in Theorem~\ref{theorem:asymmetric} that as $p_k$ approaches 1 for each agent $k$, i.e., when all agents participate synchronously at each iteration, we recover Theorem~\ref{prop:sync_collab}. Also notice that the essential difference from the synchronous scenario is the replacement of confidence weights $\pi_k$ by $\pi_k p_k$. This is intuitive since more participation by an agent is expected to increase its influence on the joint decision, as if it had a higher confidence from the server. Similarly, more participation by the other agents decreases the overall impact of an agent on the joint decision, as the ``relative'' participation of the agent is decreasing compared to the others.
\begin{theorem}[\textbf{Symmetric communication}]\label{theorem:symmetric} Under the symmetric communication protocol described in Sec.~\ref{sec:async_behavior} the expected log-belief ratio under intervention is given by
\begin{equation}\label{eq:symmetric_th_lambda}
    \widetilde{\lambda}_{\infty} (\theta) = \dfrac{1}{\pi_m p_m}\sum\limits_{k \neq m} \dfrac{\pi_k d_k (\theta)}{1-\pi_k (1-p_k)}+ \log \dfrac{\mu_{m}(\theta^{\circ})}{\mu_{m}(\theta)} 
\end{equation}    
This implies by \eqref{eq:cm_general_def} that the causal effect of agent $m$ on the joint decision is given by
\begin{align}
      C_m \!\!= 1\!-\!\dfrac{1}{1 + \!\!\sum\limits_{\theta \neq \theta^\circ} \dfrac{\mu_{m}(\theta)}{\mu_{m}(\theta^\circ)}  \exp \Big \{ \!  \dfrac{-1}{\pi_m p_m}\!\sum\limits_{k \neq m} \dfrac{\pi_k d_k (\theta)}{1-\pi_k (1-p_k)} \Big \}}
\end{align}
\end{theorem}
\begin{proof}
    Appendix~\ref{appendix:symmetric}.
\end{proof}
Similar to the asymmetric communication scenario in Theorem~\ref{theorem:asymmetric}, as $p_k \to 1$ for all agents, Theorem~\ref{theorem:symmetric} recovers the synchronous collaboration result of Theorem~\ref{prop:sync_collab}. Furthermore, as $p_m \to 0$, notice that $\widetilde{\lambda}_{\infty} (\theta) \to \infty$, which in turn implies $C_m \to 0$. In other words, if an agent does not participate in the decision making, it does not have any impact on the decision.

Next, we compare the causal impacts of agents under both asymmetric and symmetric communication schemes, given the same participation $\{p_k\}_{k=1}^K$, confidence weight $\{\pi_k\}_{k=1}^K$, and informativeness parameters $\{d_k(\theta)\}_{k=1}^K$.
\begin{corollary}[\textbf{Comparison of asynchronous scenarios}]\label{corollary:comparison}
    Agent \(m\) exerts a stronger causal impact on the joint decision in the symmetric scenario compared to the asymmetric scenario if the misinformation strength (i.e., intervened belief) satisfies 
    \begin{equation}\label{eq:th1_2_comp_cond}
    \log \dfrac{\mu_m (\theta)}{\mu_m (\theta^\circ)} \geq \sum_{k \neq m} \dfrac{\pi_k d_k (\theta)}{\pi_m (1-p_m)} \Big ( \dfrac{1}{p_m (1- \pi_k (1- p_k))} - p_k \Big )
\end{equation}
\end{corollary}
\begin{proof}
    Notice from \eqref{eq:asymmetric_th_lambda} and \eqref{eq:symmetric_th_lambda} if agent \(m\) meets condition \eqref{eq:th1_2_comp_cond},
then the \(\widetilde{\lambda}_{\infty} (\theta)\) term in \eqref{eq:asymmetric_th_lambda} exceeds that in \eqref{eq:symmetric_th_lambda}.
 The result then follows by definition \eqref{eq:expected_log_belief_trans}, since $\widetilde{\lambda}_{\infty} (\theta)$ is inversely proportional to the causal impact $C_m$.
\end{proof}

Corollary~\ref{corollary:comparison} holds significant relevance for practical applications. For our problem setting, we can define misinformation as the ratio of the belief on wrong hypothesis to true hypothesis, i.e., $\frac{\mu_m (\theta)}{\mu_m (\theta^\circ)}$.
Commonly, if misinformation is originating from \textit{malfunctioning} agents, it is moderate. In contrast, \textit{malicious} agents often supply adversarial misinformation that can be extreme. This suggests that the symmetric communication scenario is more vulnerable to highly outlying information potentially caused by adversarial agents, while asymmetric communication is more sensitive to moderate level misinformation that typically emerges from malfunctioning agents without harmful intentions. Furthermore, for a fair decision-making process that aims to account for all agents while remaining resilient against adversarial threats, asymmetric communication appears to be better in comparison to the symmetric case. This is because it allocates greater causal weight to moderate deviations from the nominal belief state while also reducing the influence of extreme misinformation, providing a safeguard against adversarial attacks.

\section{Experimental Results}\label{sec:numerical}

\subsection{Synthetic data}

To illustrate our theoretical results, we first consider a binary hypothesis testing problem with \(K=12\) agents connected to a FC, each receiving observations that follow a Gaussian distribution. In other words, two possible hypotheses underlie streaming data with same variance Gaussian distributions but different means. Under the null hypothesis, the mean for all agents is assumed to be \(0\), while under the alternative hypothesis, it is \(0.5\) for odd-indexed agents and \(1\) for even-indexed agents. The probability of participation \(p_k\), which is defined in \eqref{eq:qki_pk_def} is set to \(0.8\) for each agent $k$ with indices $1-3$, to \(0.6\) for agent indices $4-6$, \(0.4\) for agent indices $7-9$, and \(0.2\) for agent indices $10-12$. Furthermore, the confidence weight \(\pi_k\) assigned by the server to each agent \(k\) is \(0.125\) for agent indices $1-4$, \(0.075\) for agent indices $5-8$, and \(0.05\) for agent indices $9-12$, ensuring that the sum of all weights across the $K=12$ agents equals 1.

\begin{figure}[ht]
	\centering
	{%
       \includegraphics[width=0.93\linewidth]{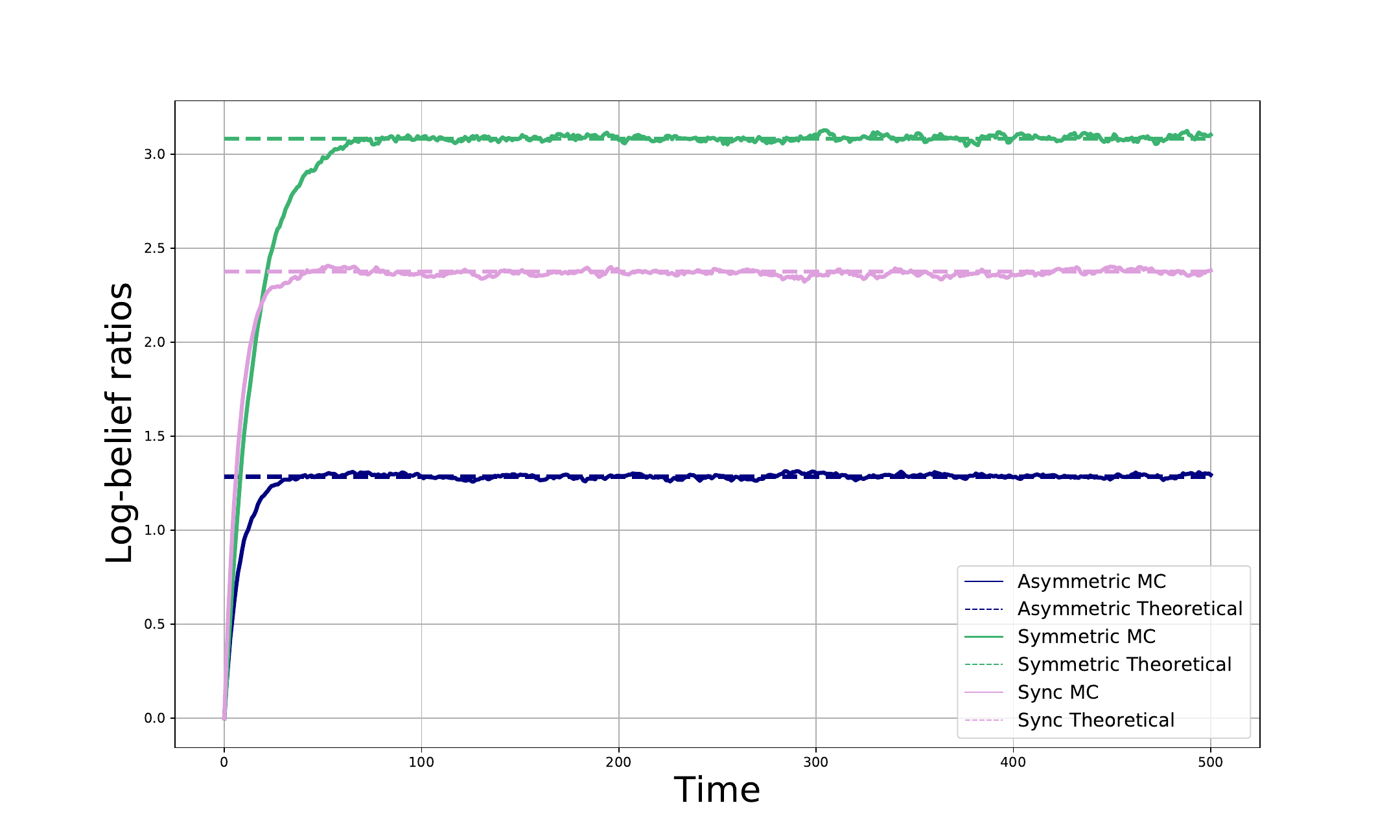}}
	\caption{\small Simulated log-belief ratios averaged over 1000 Monte Carlo (MC) simulations and theoretical expressions over time overlap with each other, verifying the derivations in Theorems~\ref{prop:sync_collab}--\ref{theorem:symmetric}.}
 \label{fig:lbr_time}
\end{figure}

In the first experiment, we average \(1000\) simulations for three settings: the synchronous setting from Sec.~\ref{sec:fed_inference}, and the asymmetric and symmetric settings from Sec.~\ref{sec:async_behavior}. This is performed under an intervention on agent \(m=1\) with uniform beliefs ($\mu_m (\theta) = 0.5$). We plot the evolution of log-belief ratios over \(500\) time instants in Fig.~\ref{fig:lbr_time}, as well as the derived expressions for these values from Theorems \ref{prop:sync_collab}, \ref{theorem:asymmetric}, and \ref{theorem:symmetric}. Notice that the simulated log-belief ratios verify the derived analytical results since they closely align with the theoretical expressions.

\begin{figure}[ht]
	\centering
	{%
        \includegraphics[width=0.88\linewidth]{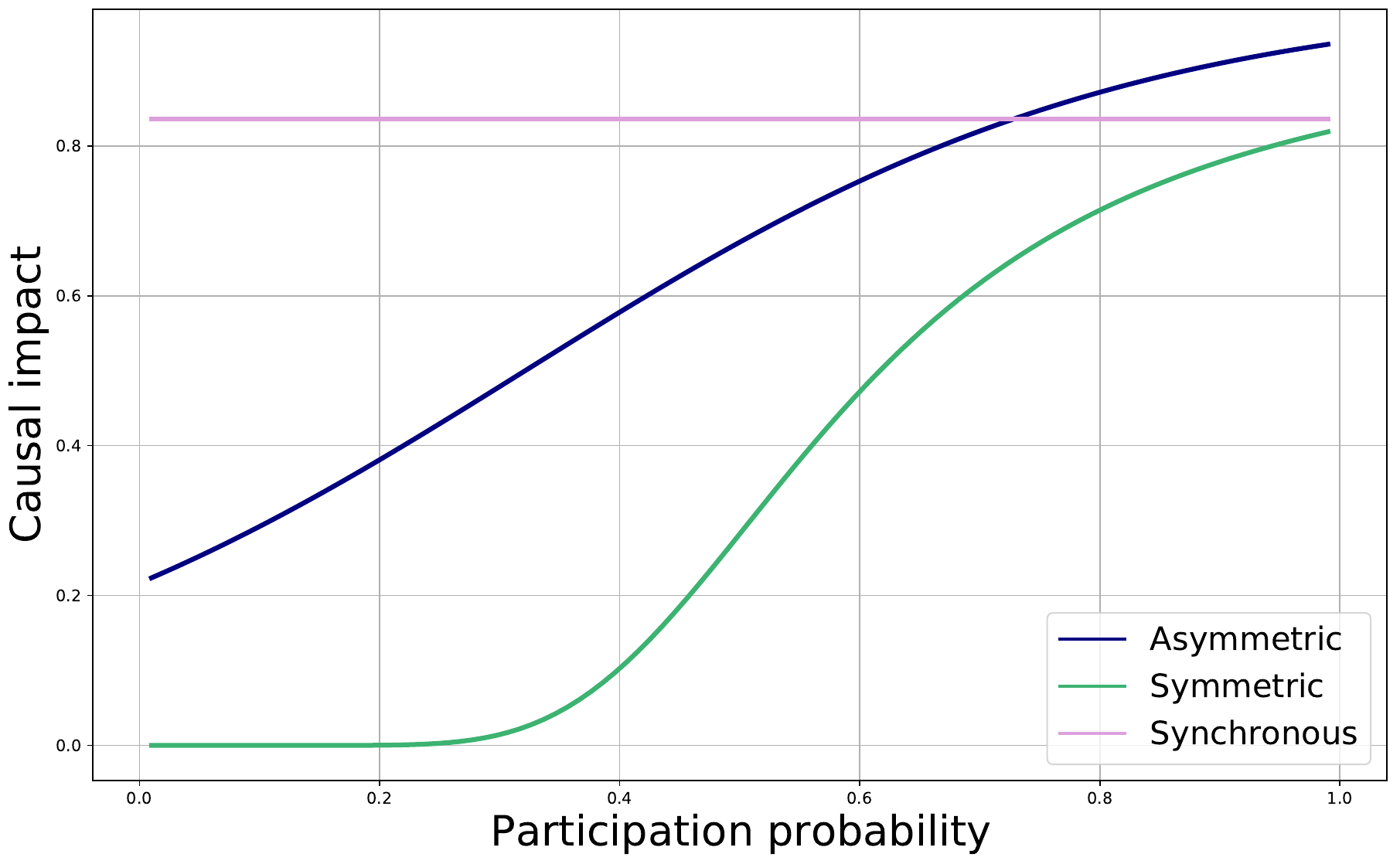}}
	\caption{\small Causal impact of agent $m=1$ on the joint decision with changing participation probability $p_m$. Note that $p_m$ is constant for the synchronous case, and the corresponding constant line is also provided in the plot for comparison purposes.}
    \label{fig:cm_pm}
\end{figure}

In Fig.~\ref{fig:cm_pm}, we illustrate the causal impact of agent \(m=1\) on the joint decision with respect to changing participation probability \(p_m\). We also include the synchronous setting where all agents participate with probability $\{p_k\}_{k=1}^K = 1$ as a reference. It is evident from this figure that increasing the frequency of information transmission by an agent increases its impact on the collaborative decision.

\begin{figure}[ht]
	\centering
	{%
       \includegraphics[width=0.9\linewidth]{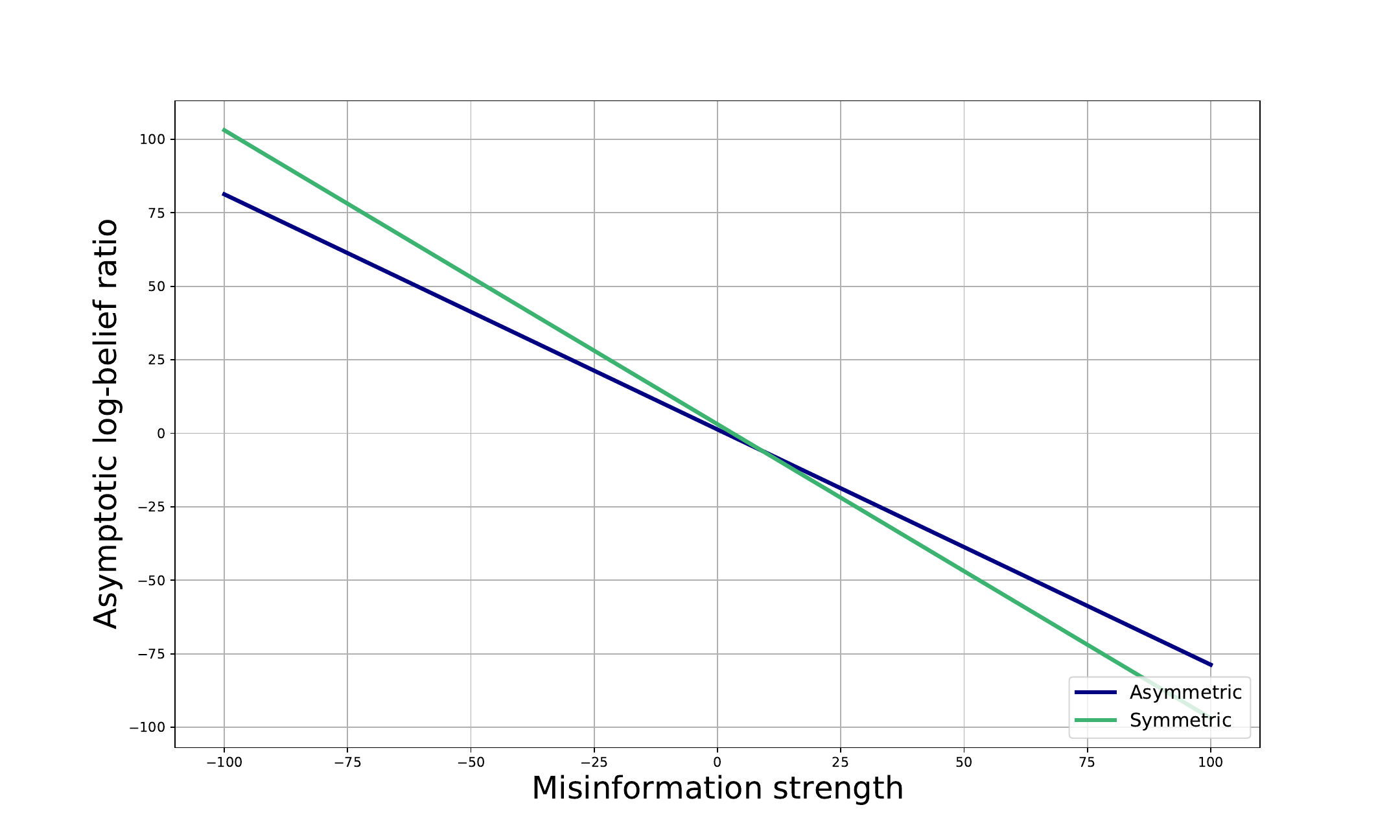}}
	\caption{\small Asymptotic log-belief ratio with respect to misinformation strength the \(\log \dfrac{\mu_m (\theta)}{\mu_m (\theta^\circ)}\), verifying the derived threshold in Corollary~\ref{corollary:comparison}.}
	\label{fig:lbr_misinfo}
\end{figure}

Next, in Fig.~\ref{fig:lbr_misinfo}, we plot the asymptotic log-belief ratios in relation to varying intervention strengths \(\log \frac{\mu_m (\theta)}{\mu_m (\theta^\circ)}\) on agent \(m=1\). Supporting our theoretical result in \eqref{eq:th1_2_comp_cond}, the log-belief ratio in the asymmetric setting surpasses the one in the symmetric setting when the misinformation strength exceeds a certain threshold. As discussed before, this means that under conditions of high misinformation supply, the asymmetric communication framework assigns a relatively smaller causal impact compared to the symmetric communication framework.

\begin{figure}[ht]
	\centering
	\includegraphics[width=.97\linewidth]{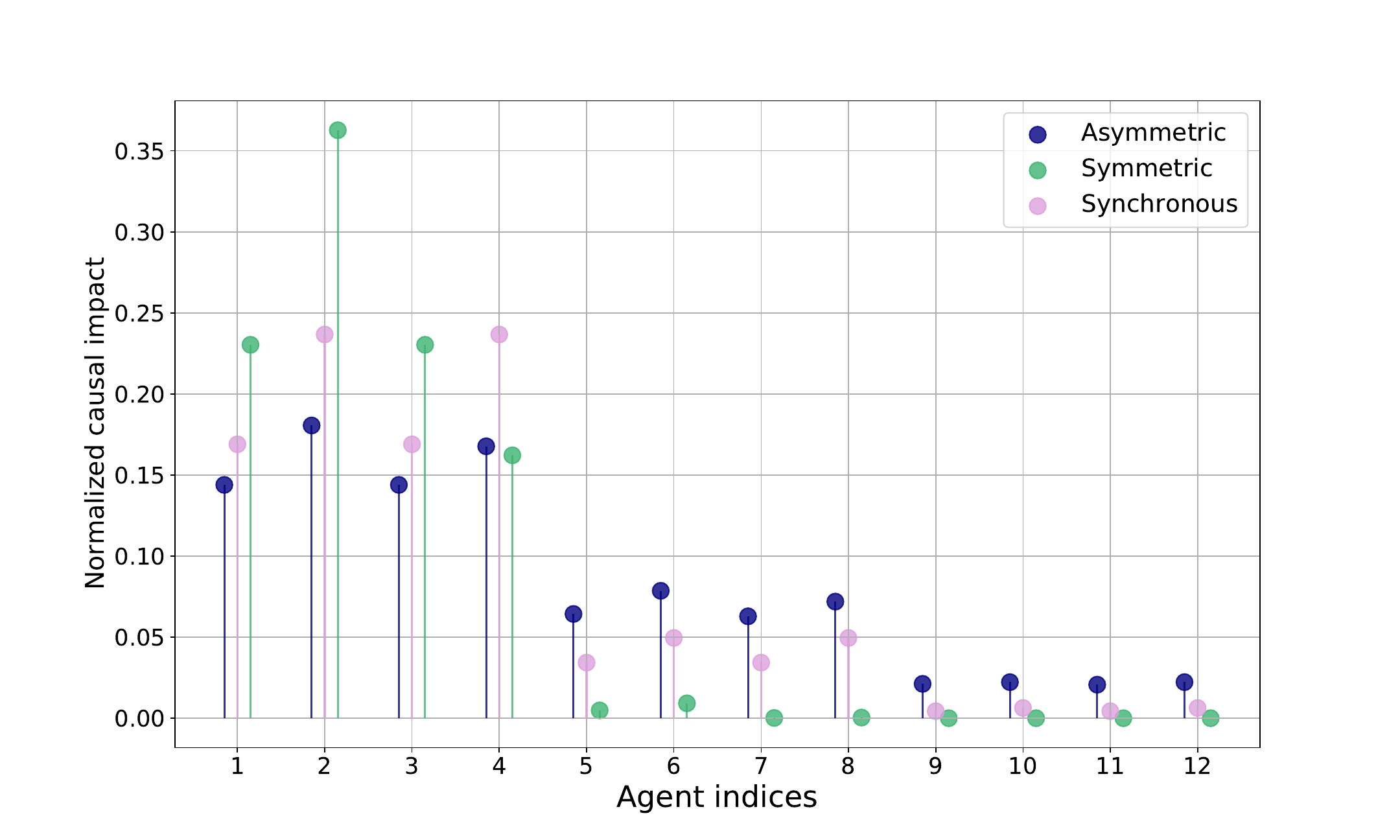}
	\caption{\small Causal impacts of each agent over three scenarios. The scores are normalized such that for each scenario, agents' scores sum up to one. It is clear that the distribution of the scores for the symmetric case has a higher skewness, as suggested by the theoretical results.}
	\label{fig:norm_ranking}
\end{figure}

Finally, in Fig.~\ref{fig:norm_ranking} we plot the causal impact of each agent on the joint decision which are normalized such that the sum of agents' impacts under each strategy equals to 1. This plot reveals that the asymmetric communication protocol results in a more uniform distribution of impacts, whereas the symmetric communication approach leads to a few agents having significant influence on the joint decision. This supports our discussions, suggesting that asymmetric communication fosters a fairer decision-making process that assigns a more uniform impact over participating agents under moderate deviations.

\begin{figure*}[] 
    \centering
  \subfloat[\label{fig:real_sensor_1}]{%
       \includegraphics[width=0.31\linewidth]{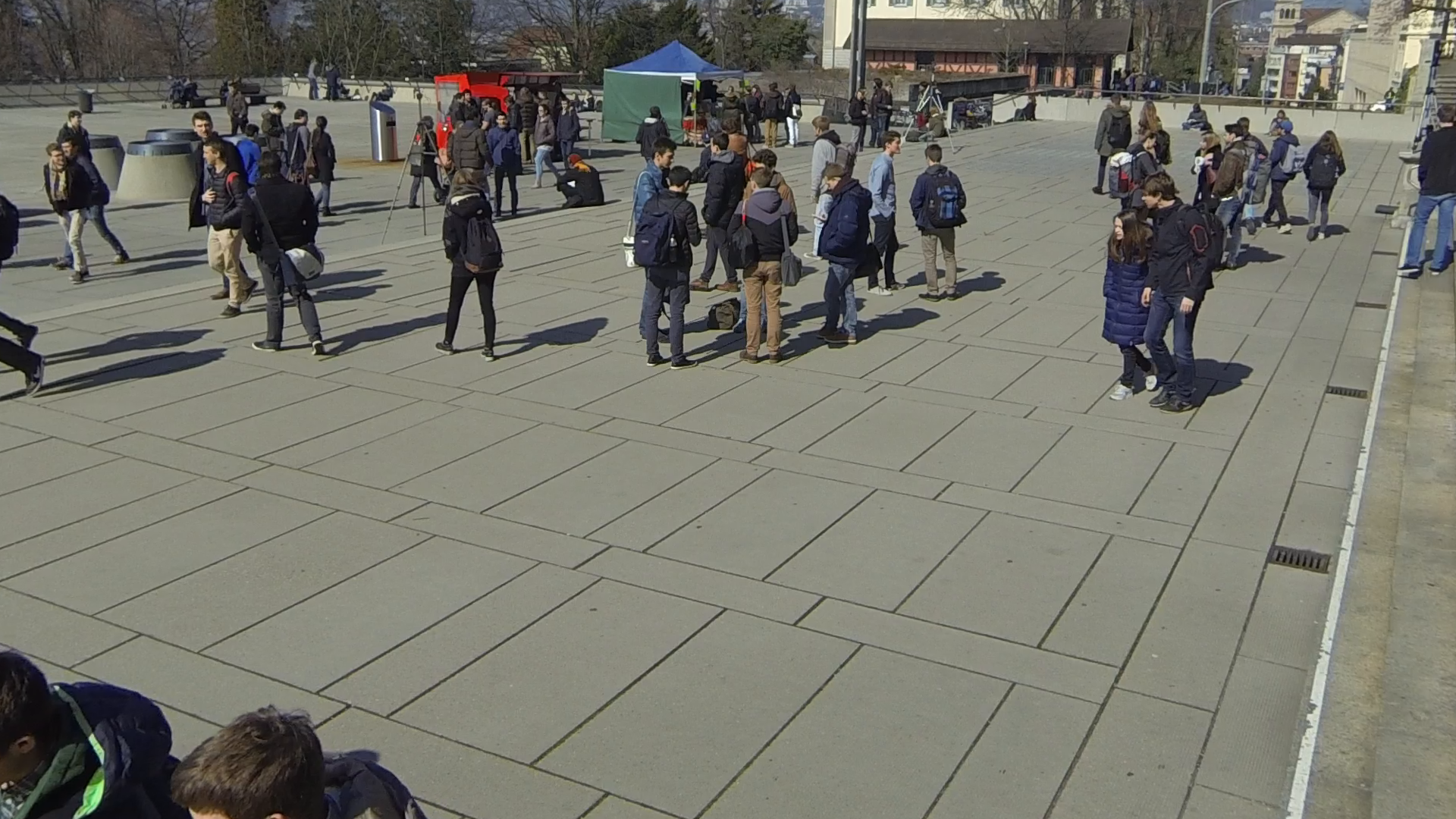}}
    \hfill
  \subfloat[\label{fig:real_sensor_3}]{%
        \includegraphics[width=0.31\linewidth]{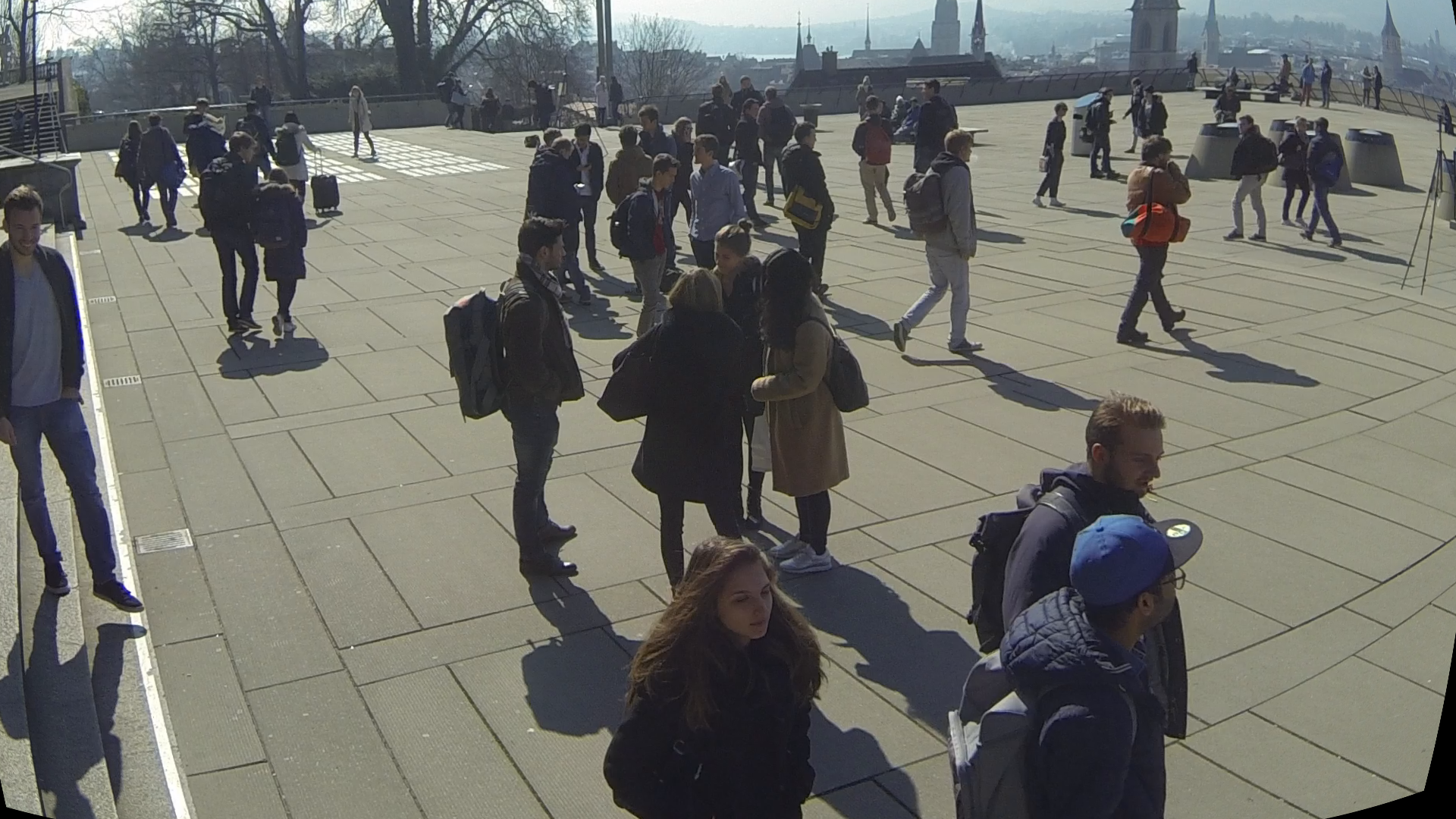}}
        \hfill
  \subfloat[\label{fig:real_sensor_6}]{%
       \includegraphics[width=0.31\linewidth]{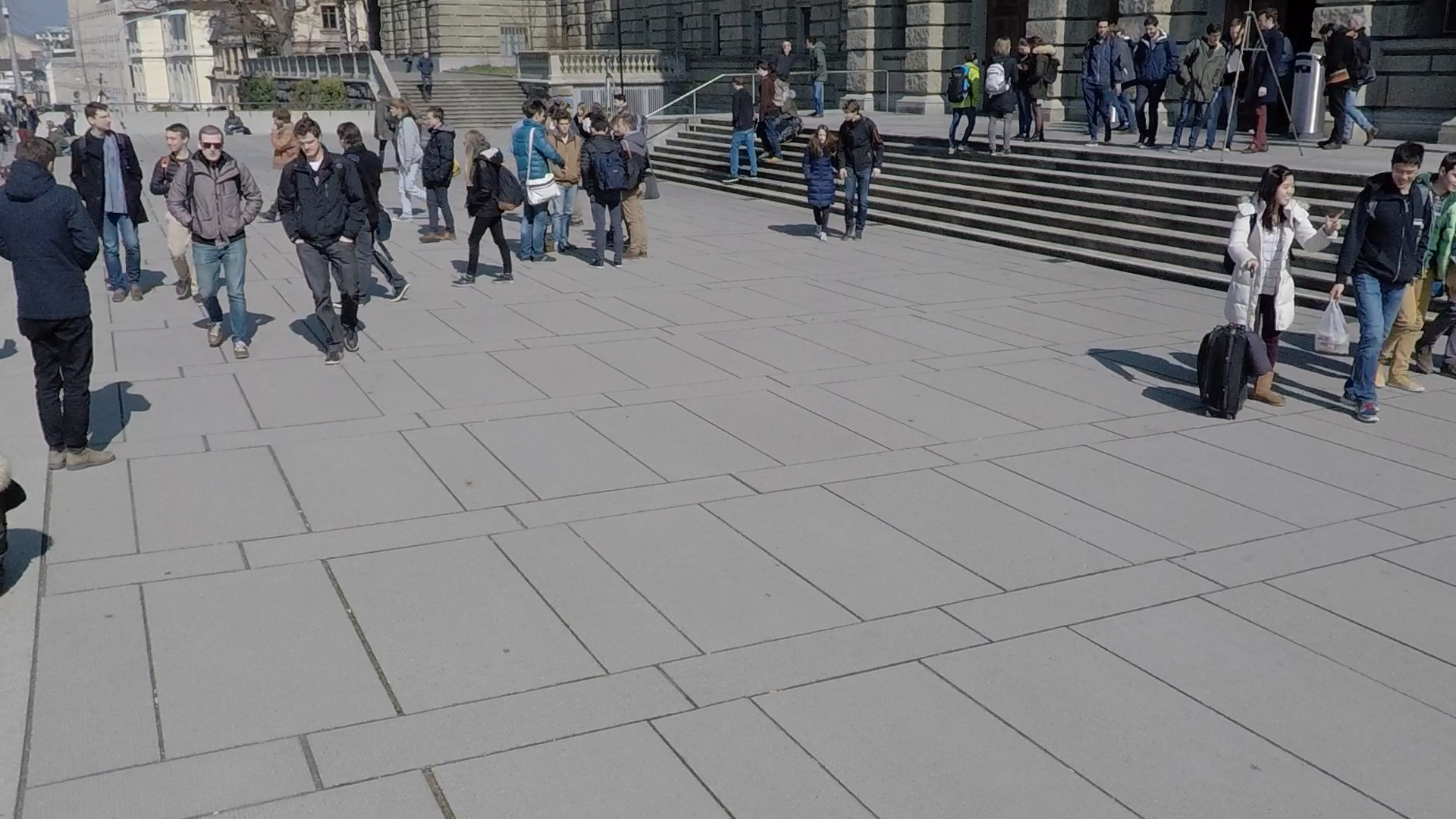}}
  \caption{\small Images are part of the WILDTRACK dataset \cite{chavdarova_wildtrack}, which is acquired in front of the main building of ETH Zurich, Switzerland. In total, the dataset contains $7$ simultaneous image sequences (with a rate of $60$ frames per second), where each image has a resolution of $1920 \times 1080$ pixels. The sample figures $(a)$, $(b)$, and $(c)$ from the dataset capture the same area simultaneously from different perspectives by agents $1$, $3$, and $7$, respectively.}
  \label{fig:photo_samples} 
\end{figure*}

\subsection{Application: Multi-camera crowd counting}

Next, we apply our results to a multi-view crowd-size estimation application using the WILDTRACK dataset from \cite{chavdarova_wildtrack}. This dataset consists of synchronized video frames captured by seven static cameras (functioning as agents in our model, $K=7$) with overlapping fields of view --- see Fig.~\ref{fig:photo_samples} for sample images. 

The primary goal of the agents is to cooperatively track the dynamic size of the crowd in a specific overlapping region observed by all cameras. For this particular application, the aforementioned variables in the paper correspond to the following:
\begin{itemize}
    \item For each agent $k$ (a camera), observation $\bxi_{k,i}$ corresponds to that agent's own recorded image frame at time instant $i$. Note that the cameras record the environment with 60 frames per second, that is, $60$ time instants correspond to one second in total.
    \item A hypothesis $\theta \in \Theta$ is a possible integer for the crowd size, and $\Theta = \{0,1, \dots, 50\}$ is the set of all possible hypotheses. For the current application at hand, it is known that the number of people in the region of interest will not surpass $50$.
    \item To apply Algorithms~\ref{alg:sync_alg}--\ref{alg:symmetric}, we equip the agents with the pre-trained crowd counting neural network from \cite{liu_2019_CVPR}. We then calibrate the likelihood functions of the agents by using the neural network estimates as well some dataset specific samples in order to obtain $L_k(\bxi_{k,i}|\theta)$ for each $\theta \in \Theta$ and for all agents. 
    \item We set the weights FC assigns to the $K=7$ agents uniform, i.e., $\pi_k = \frac{1}{7}$ for each agent $k$. Moreover, for both asynchronous scenarios, the participation probabilities $p_k$ are set at $0.5$.
\end{itemize}

\begin{figure}[ht]
	\centering
	\includegraphics[width=.94\linewidth]{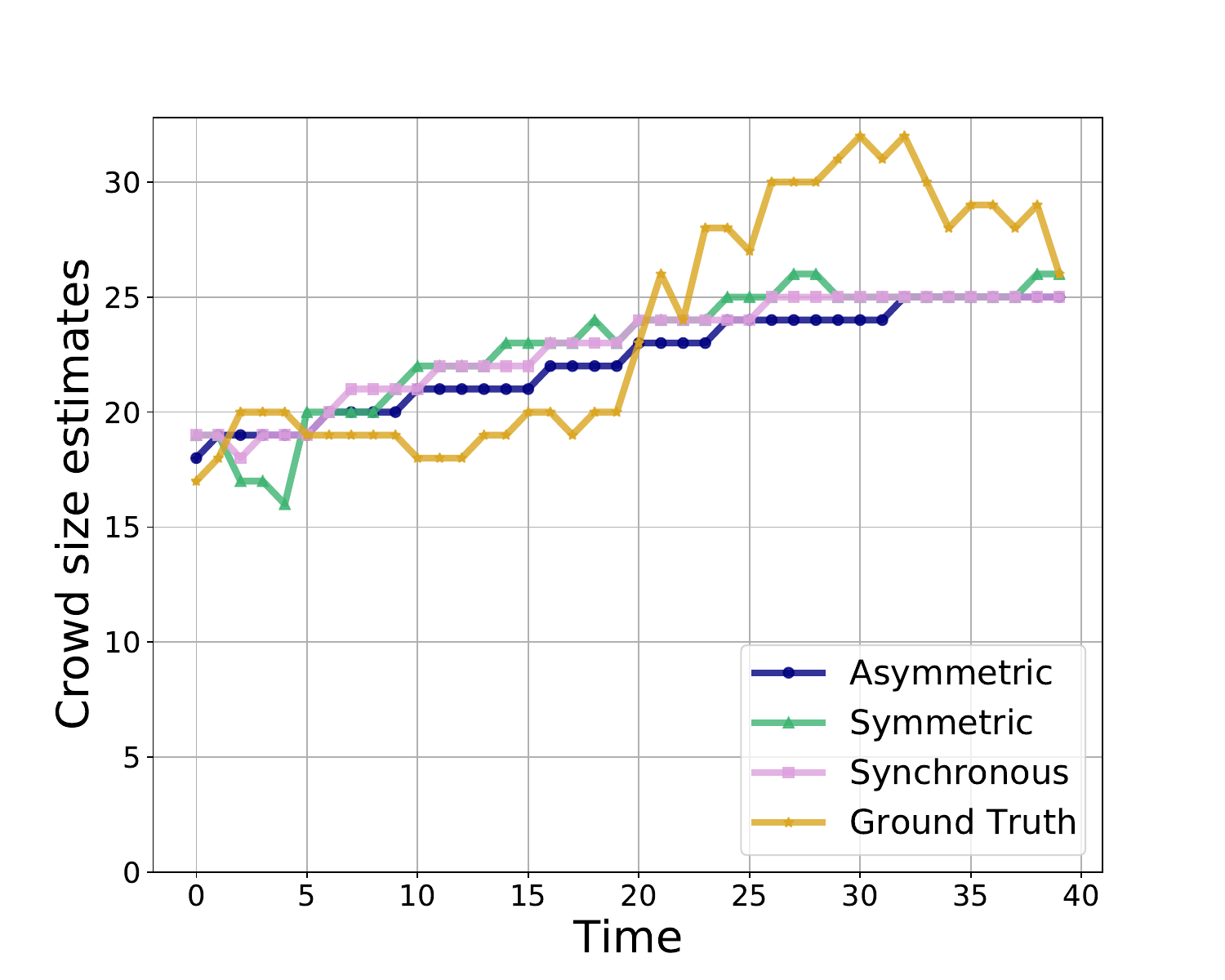}
	\caption{\small Crowd count estimates of FC under all three scenarios, along with the true number of people (ground truth). The estimates of FC correspond to the hypothesis that maximizes the belief at each time instant.}
	\label{fig:real_estimates}
\end{figure}

Under these parameters, Fig.~\ref{fig:real_estimates} illustrates the FC's crowd count estimates under all three scenarios, along with the actual number of people present (ground truth). Here, the estimates of FC represent the hypothesis $\theta$ that maximizes the belief $\bmu_{i} (\theta)$ at each time instant $i$. Furthermore, in Fig.~\ref{fig:real_norm_ranking}, the normalized causal impact scores of each camera on the joint decision are presented for all three scenarios using uniform beliefs intervention. Notably, the score distribution in the symmetric case exhibits the highest level of skewness, which mirrors the conclusion with the synthetic data in Fig.~\ref{fig:norm_ranking}.

\begin{figure}[ht]
	\centering
	\includegraphics[width=.99\linewidth]{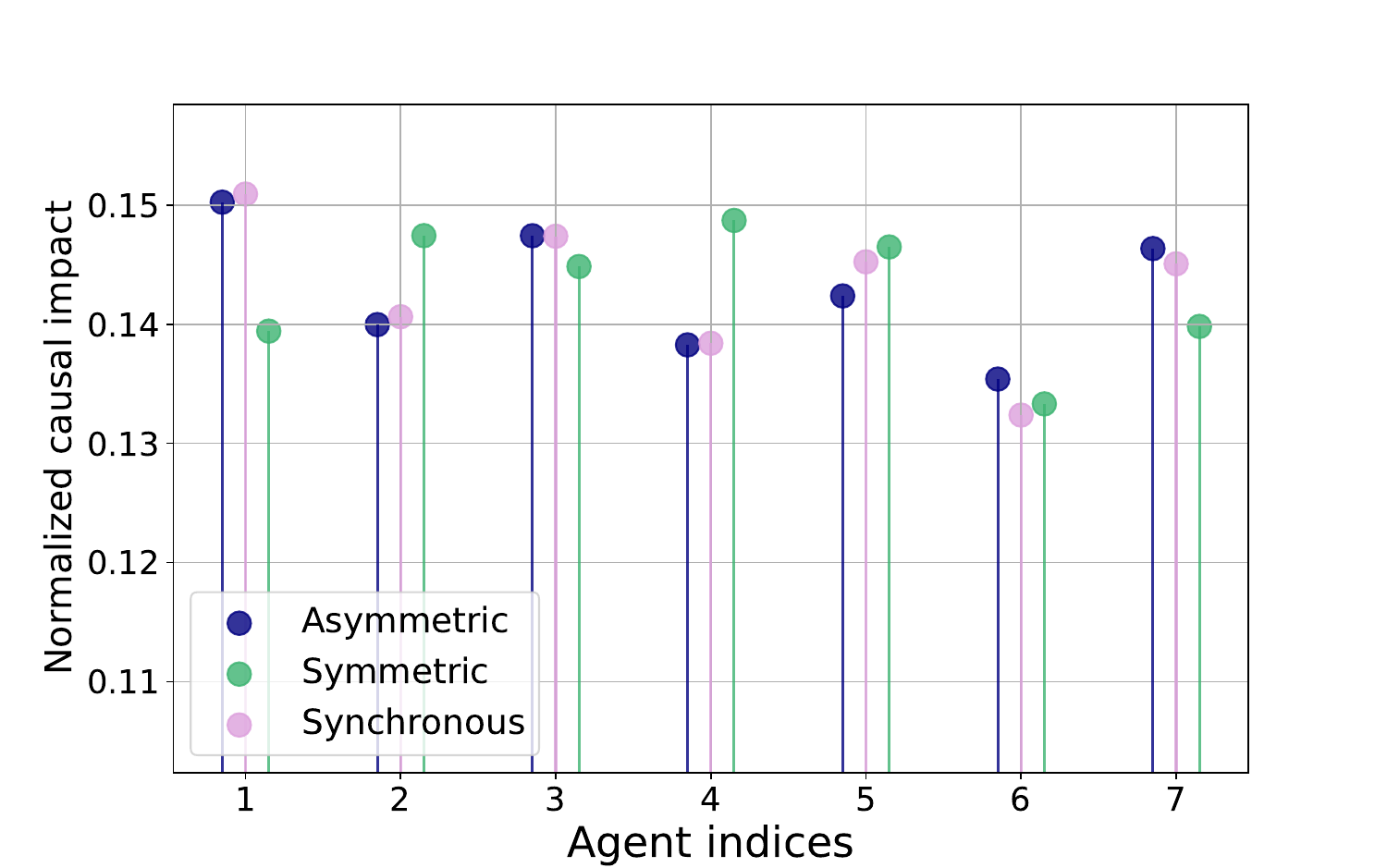}
	\caption{\small Normalized causal impact scores of each camera on the joint decision for all three scenarios. As in the Fig.~\ref{fig:norm_ranking} for synthetic data, distribution of the scores for the symmetric case has the highest skewness.}
	\label{fig:real_norm_ranking}
\end{figure}

\section{Conclusion}

In this paper, we examined a collaborative prediction framework for identifying and quantifying causal impact of an agent where agents exchange their local inferences about a common target variable with a fusion center. We incorporated two asynchronicity scenarios that differ in terms of whether the fusion center updates the agents that do not provide information. Utilizing a causal theoretical framework, we derived expressions that describe how each agent's impact on the collective decision varies based on factors such as the distribution of data (via KL divergences representing the informativeness of data) received by the agents and their participation frequencies. 

The results reveal that an agent has a stronger impact on the joint decision in the symmetric (reciprocal) communication protocol compared to the asymmetric communication protocol if the misinformation strength surpasses some threshold. This implies that asymmetric communication protocols are more robust in the face of adversarial attacks. Nevertheless, symmetric communication offers greater resilience to moderate deviations from the usual, such as in the case of malfunctioning agents without harmful intentions.

Future directions include extending the causal impact analysis on this federated framework to decentralized peer-to-peer networks, and also examining different decision aggregation strategies at the server side such as median-based robust fusion \cite{zoubir2018robust}. Moreover, the current work is based on the traditional social learning framework which is reported to make the agents stubborn. Extending this work to adaptive agents \cite{bordignon2021adaptive} is another interesting direction to pursue.

\appendices

\section{Proof of Theorem~\ref{theorem:no_intervention}}\label{appendix:pre_int}

In this section, we prove that under both scenarios considered for asynchronous behavior, and without any intervention, the expected beliefs at the agents place the value $1$ on the true hypothesis as $i \to \infty$. Note that the proof for the synchronous communication case is already exists in the literature \cite{kayaalp2023fusion,lalitha_2018,nedic_2017,bordignon2021adaptive}. It can also be recovered from our novel derivations here by setting $p_k \to 1$ for each agent $k$.
\subsection{Asymmetric communication}
Recall that $\bq_{k,i}$ is the Bernoulli random variable that is equal to $1$ if agent $k$ is connected to the FC at time $i$ and is sending information, i.e.,
\begin{equation}
\bq_{k,i} \!=\! \begin{cases}
    1,\ \text{if agent $k$ is sending information to FC at time $i$}\\
    0,\ \text{else}.
\end{cases}
\end{equation}
Define the scalar random variables
\begin{equation}
    \blambda_{i} (\theta) \triangleq \log \dfrac{\bmu_{i} (\theta^\circ) }{\bmu_{i} (\theta) }, \quad\bx_{k,i}(\theta) \triangleq \log \dfrac{ L_k(\bxi_{k,i} | \theta^\circ)}{ L_k(\bxi_{k,i} | \theta)}.
\end{equation} 
In this case, to derive the recursion of log-belief ratios, observe from \eqref{eq:dif_adapt_step} that if an agent $k$:
\begin{itemize}
    \item is not sending any information at time instant $i$ (i.e., it is idle), then the FC uses its own log-belief ratio from the previous time instant $\blambda_{i-1} (\theta)$ to fill the missing information of agent $k$ during aggregation.
    \item Otherwise, if agent $k$ is sending its intermediate belief to the server, then its contribution on the FC decision is $\blambda_{i-1}(\theta) + \bx_{k,i}(\theta)$.
\end{itemize}
In light of the observation above, the contribution of each agent $k$ at time $i$ is a function of $\bq_{k,i}$ and can be written as
\begin{align}\label{eq:lbr_contr_two_way_ni_1}
\blambda_{i-1}(\theta),\quad \bq_{k,i} = 0\\
\blambda_{i-1}(\theta) + \bx_{k,i},\quad \bq_{k,i} = 1
\end{align}
which is equivalent to
\begin{align}\label{eq:lbr_contr_two_way_ni}
\blambda_{i-1}(\theta)(1-\bq_{k,i}) + (\blambda_{i-1}(\theta) &+ \bx_{k,i})\bq_{k,i} \notag \\
&= \blambda_{i-1}(\theta) + \bx_{k,i}\bq_{k,i}.
\end{align}
Under the information fusion rule \eqref{eq:geometric_fusion}, the fusion center update becomes
\begin{equation}\label{eq:lbr_update_two_way_ni}
 \blambda_i(\theta) = \blambda_{i-1}(\theta) + \brma_i^{\T}\bx_{i}(\theta),
\end{equation}
where we are introducing the vectors
\begin{equation}
    \brma_i \triangleq [\pi_1\bq_{1,i},\ldots,\pi_k\bq_{K,i}]^{\T}
\end{equation}
and 
\begin{equation}
    \brmx_i \triangleq [\bx_{1,i},\dots,\bx_{K,i}]^{\T}.
\end{equation}
Taking expectations over randomness of data yields
\begin{align}
 \lambda_i(\theta) & \triangleq \E[\blambda_i(\theta)]\notag \\ &\stackrel{\eqref{eq:lbr_update_two_way_ni}}{=} \E[\blambda_{i-1}(\theta) + \brma_i^{\T}\brmx_{i}(\theta)]
 \notag \\
 &= \lambda_{i-1}(\theta) + \rma^{\T}\rmd(\theta),
\end{align}
with the definitions:
\begin{align}
   \rma &\triangleq [\pi_1 p_1, \ldots, \pi_K p_K]^{\T}, \notag \\
     \rmd(\theta) &\triangleq \Big [d_1(\theta), \ldots, d_K (\theta) \Big ]^{\T}, \notag \\     d_k (\theta) &\triangleq \dkl \big (L_k(\cdot|\theta^{\circ}) || L_k(\cdot|\theta) \big).
\end{align}
The global identifiability assumption ensures that for each $\theta \neq \theta^\circ$, there exists at least one agent $k^\star$ that satisfies $d_{k^\star}(\theta) > 0$. Therefore, for each wrong hypothesis $\theta \neq \theta^\circ$, it holds that $\lambda_i(\theta) \to + \infty$ as long as $\pi_{k^\star} p_{k^\star} > 0$, which in turn means $\mu_{i}(\theta^\circ) \to 1$ for the collective decision of agents as $i \to \infty$. 

\subsection{Symmetric communication}
In this scenario, agents do not get updated from the FC if they do not send information to the FC. Therefore, at each time instant, they can have different beliefs and priors. This situation necessitates the study of the evolution of a local belief (prior) $\bmu_{k,i}$ and the corresponding local log-belief ratio (LBR) $\blambda_{k,i}$. If we define a vector consisting of all LBRs from all agents as
\begin{equation}                                                                                               
   \bLambda_i (\theta) \triangleq [\blambda_{1,i} (\theta),\ldots,\blambda_{K,i} (\theta)]^{\T} ,
\end{equation}
then, the LBR at the server side evolves according to the following dynamics:
\begin{equation}\label{eq:lbr_ss_2}
 \blambda_i(\theta) = \brma_i^{\T}(\bLambda_{i-1}(\theta) + \brmx_{i}(\theta)) + (\bar{\brma}_i^{\T} \mathds{1}_K) \blambda_{i-1}(\theta),
\end{equation}
where we use the bar notation to denote additive complements of the random variables, i.e., 
\begin{equation}
\bar{\brma}_i \triangleq \pi - \brma_i = [\pi_1 \bar\bq_{1,i}, \ldots, \pi_K \bar\bq_{K,i}]^{\T}, \:\: \bar\bq_{k,i} \triangleq 1- \bq_{k,i}
\end{equation}
and also use the definition
\begin{equation}
    \brmx_{i}(\theta) \triangleq [\bx_{1,i} (\theta), \bx_{2,i} (\theta), \dots, \bx_{K,i} (\theta)]^{\T}.
\end{equation}
Taking the expectation of both sides in \eqref{eq:lbr_ss_2} yields
\begin{equation}\label{eq:c2_ni_server_der}
 \lambda_i(\theta) = \rma^{\T}(\Lambda_{i-1}(\theta) + \rmd(\theta)) + (\bar{\rma}^{\T} \mathds{1}_K) \lambda_{i-1}(\theta).
\end{equation}
On the other hand, the LBR evolution for an agent $k$ depends on whether it provides information to the FC update in \eqref{eq:lbr_ss_2}, and is given by
\begin{equation}
\blambda_{k,i}(\theta) = \blambda_i(\theta)\bq_{k,i} + (\blambda_{k,i-1}(\theta) + \bx_{k,i}(\theta))\bar\bq_{k,i}.
\end{equation}
By taking expectations of both sides we arrive at
\begin{equation}\label{eq:c2_ni_der_agent}
    \lambda_{k,i}(\theta) = \E[\blambda_i(\theta)\bq_{k,i}] + (\lambda_{k,i-1}(\theta) + d_{k}(\theta))\bar p_{k},
\end{equation}
where $\bar p_{k} \triangleq 1 - p_{k}$. Note that in general,
\begin{equation}
    \E[\blambda_i(\theta)\bq_{k,i}] \neq  \E[\blambda_i(\theta)]\E[\bq_{k,i}]
\end{equation}
due to the information sharing of server is conditioned on agents' information sharing, and that they are not independent. However, it holds that
\begin{align}
&\E[\blambda_i (\theta)\bq_{k,i}] \notag \\&= \E\bigg[\bq_{k,i}\bigg( \sum_{\ell =1}^K \pi_\ell(\blambda_{\ell,i-1} (\theta) + \bx_{\ell,i} (\theta))\bq_{\ell,i} \notag \\ & \qquad \qquad + \sum_{\ell=1}^K \pi_\ell \bar\bq_{\ell,i}\blambda_{i-1} (\theta)
\bigg)\bigg] \notag \\
&= p_k\bigg(\pi_k(\lambda_{k,i-1} (\theta) + d_k (\theta)) \notag \\ &\qquad\quad +\sum_{\ell \neq k}\pi_\ell(\lambda_{\ell,i-1} (\theta) + d_{\ell} (\theta))p_{\ell} + \sum_{\ell \neq k}\pi_\ell\bar p_{\ell} \lambda_{i-1} (\theta) \bigg).
\end{align}
Next, we introduce the variables
\begin{align}
    \rms &\triangleq [s_1,\dots, s_K]^{\T},\ s_k \triangleq p_k\sum_{\ell \neq k} \pi_\ell\bar p_\ell,  \\
    \rmp & \triangleq [p_1, \dots, p_K]^{\T},\ \bar{\rmp} \triangleq \mathds{1}_K-\rmp\\
    \sigma & \triangleq \sum_{k=1}^K \pi_k \bar p_k = \bar{a}^{\T} \mathds{1}_K
\end{align}
and the $K \times K$ diagonal matrices 
\begin{equation}
        \rmA \triangleq \diag(\rma), \quad \rmP \triangleq \diag(\rmp), \quad \bar{\rmP} \triangleq \diag(\bar{\rmp}) .
\end{equation}
Accordingly, if we also define the $(K+1)$ dimensional extended LBR vector as
\begin{equation}
\bar \Lambda_i(\theta) \triangleq \left[\begin{array}{@{}c}
 \:\: \Lambda_i(\theta)
 \\
   \:\: \lambda_i(\theta)
\end{array}\right],
\end{equation}
then, by relations \eqref{eq:c2_ni_server_der} and \eqref{eq:c2_ni_der_agent}, we arrive at a linear recursion of the following form:
\begin{equation}\label{eq:linear_rec_ni_c2}
    \bar \Lambda_i(\theta) = \rmR \bar \Lambda_{i-1}(\theta) + \rmU \rmd(\theta).
\end{equation}
Here, we introduced the $(K+1) \times (K+1)$ dimensional matrix
\begin{align}
\rmR  &\triangleq 
 \left(\begin{array}{@{}c|c@{}}
  \bar \rmP \rmA + \bar {\rmP} + \rmp\rma^{\T}
  & \: \: \rms \\ 
\hline 
  \rma^{\T} &
  \: \:  \sigma
\end{array}\right)
\end{align}
and the $(K+1) \times K$ dimensional matrix
\begin{align}
\rmU  &=\left(\begin{array}{@{}c}
  \bar \rmP \rmA + \bar{\rmP} + \rmp\rma^{\T}
 \\
\hline
  \rma^{\T}
\end{array}\right).
\end{align}
Since $\rmR$ is a stochastic matrix with nonzero entries, it is also a primitive matrix. Therefore, $\rmR$ has a Perron eigenvector $\pi_{\rmR}$ with all positive entries that corresponds to the largest magnitude eigenvalue. Hence, it holds that 
\begin{equation}
\frac 1 i \bar \Lambda_i(\theta) \to \pi_{\rmR} \rmU \rmd. 
\end{equation}
Under global identifiability assumption that for at least one agent $d_k (\theta) >0$, expected beliefs at the agents place the value 1 on the true hypothesis. 

\section{Proof of Theorem~\ref{theorem:asymmetric}}\label{appendix:asymmetric}
Under an intervention on agent $m$, the contribution of any agent $k \neq m$ at time $i$ is the same as in the pre-intervention case and is given by \eqref{eq:lbr_contr_two_way_ni}. Furthermore, the intervened agent $m$'s one-step contribution is equal to the intervention strength if agent $m$ is present, and equal to the FC's log-belief ratio (LBR) otherwise. In other words, agent $m$'s contribution under an intervention becomes
\begin{equation}\label{eq:two_way_contr_inter}
\wblambda_{i-1}(\theta)(1-\bq_{m,i}) + c\bq_{m,i}.
\end{equation}
In \eqref{eq:two_way_contr_inter}, we use $\sim$ on top of pre-intervention variables to denote the post-intervention counterparts of those variables and also note that we are defining
\begin{equation}\label{eq:constant_lbr_def}
    c \triangleq  \log \frac{\mu_{m}(\theta^{\circ})}{\mu_{m}(\theta)}
\end{equation}
for brevity of notation. Aggregating the contributions \eqref{eq:lbr_contr_two_way_ni} and \eqref{eq:two_way_contr_inter} of each agent according to the geometric averaging rule \eqref{eq:geometric_fusion} yields the following update for the LBR of the fusion center:
\begin{equation}\label{eq:lbr_update_two_way}
 \wblambda_i(\theta) = (1-\pi_m)\wblambda_{i-1}(\theta) + \brma_i^{\T}\widetilde{\brmx}_{i}(\theta),
\end{equation}
where we introduced the log-belief ratio counterpart vector under intervention:
\begin{equation}
    \widetilde{\brmx}_{i}(\theta) \triangleq [\bx_{1,i} (\theta),\ldots,\bx_{m-1,i} (\theta), c,\bx_{m+1,i}(\theta),\ldots,\bx_{K,i}(\theta)]^{\T}
\end{equation}
where $c$ is taken from \eqref{eq:constant_lbr_def}.
According to these definitions, for the expected LBR, it holds that
\begin{align}
 \wlambda_i(\theta) &\stackrel{\eqref{eq:lbr_update_two_way}}{=} \E\left[(1-\pi_m)\wblambda_{i-1}(\theta) + \brma_i^{\T}\widetilde{\brmx}_{i}(\theta) \right]
 \notag \\
  &= (1-\pi_m)\wlambda_{i-1} (\theta) + \rma^{\T} \E [\widetilde{\brmx}_{i}(\theta)] \notag \\
 &= (1-\pi_m)\wlambda_{i-1} (\theta) + \rma^{\T}\widetilde{\rmd}(\theta)
\end{align}
where the vector of KL divergences under an intervention on agent $m$ is defined as
\begin{equation}
    \widetilde{\rmd} (\theta) \triangleq [d_1 (\theta),\ldots,d_{m-1}(\theta), c,d_{m+1}(\theta),\ldots,d_{K}(\theta)]^{\T}.
\end{equation}
Consequently, in the limit, the LBR of the server under an intervention on agent $m$ is given by
\begin{align}
\lim_{i \to \infty} \wlambda_i(\theta) &= \frac 1 {\pi_m}  \rma^{\T}\widetilde{\rmd}(\theta) \notag \\ &= \frac 1 {\pi_m} \sum_{k \neq m} \pi_{k} p_{k} d_{k} (\theta) +  p_m c .
\end{align}

\section{Proof of Theorem~\ref{theorem:symmetric}}\label{appendix:symmetric}

For simplicity of notation and without loss of generality, we intervene on agent $m=1$. This means that whenever it is active ($\bq_{1,i} = 1$), the contribution in terms of the log-belief ratio will be equal to $c$ defined in \eqref{eq:constant_lbr_def}. Under such scheme, the linear recursion from \eqref{eq:linear_rec_ni_c2} transforms, under the intervention on $m=1$, to
\begin{equation}
\widetilde{\Lambda}_i (\theta) = \widetilde{\rmR} \widetilde{\Lambda}_{i-1} (\theta) + \widetilde{\rmU} \widetilde{\rmd}(\theta)
\end{equation}
where $\widetilde{\rmR}$ is the submatrix of $\rmR$ without the first column and row, $\widetilde{\rmU}$ is the submatrix of $\rmU$ without the first row, and $\widetilde{\rmd} (\theta) = [c, d_2 (\theta),\dots, d_k (\theta)]$. Note that the largest eigenvalue of $\widetilde{\rmR}$ is smaller than 1, hence we obtain
\begin{align}
    \widetilde{\Lambda}_\infty (\theta) &= (\rmI + \widetilde{\rmR} + \widetilde{\rmR}^2 \dots) \widetilde{\rmU} \widetilde{\rmd} (\theta) \notag \\
    &= (\rmI- \widetilde{\rmR} )^{-1}\widetilde{\rmU} \widetilde{\rmd}.
\end{align}
We therefore need to invert $(\rmI-\widetilde{\rmR})$. We write $\rmI-\widetilde{\rmR}$ in block matrix form:
\begin{align}
\rmM &\triangleq \rmI- \widetilde{\rmR} \notag \\
&= \left(\begin{array}{@{}c|c@{}}
  {\widetilde{\rmP}} - (\rmI-{\widetilde{\rmP}}) \widetilde{\rmA} - {\widetilde{\rmp}}\widetilde{\rma}^{\T}
  & -\widetilde{\rms} \\
\hline
  -\widetilde{\rma}^{\T} &
  1-\sigma
\end{array}\right) \notag \\
&\triangleq
\left(\begin{array}{@{}c|c@{}}
  \rmM_{11}
  & \rmM_{12} \\
\hline
  \rmM_{21}&
  \rmM_{22}
\end{array}\right)
\end{align}
where 
\begin{equation}
    \widetilde \rmp \triangleq [p_2,\ldots,p_K], \quad \widetilde \rma \triangleq [\pi_2p_2,\ldots,\pi_Kp_K], 
\end{equation}
and $\rmM_{11}, \rmM_{12}, \rmM_{21}, \rmM_{22}$ are submatrices of dimensions $(K-1) \times (K-1), (K-1)\times 1 , 1\times (K-1), 1 \times 1,$ respectively, and 
\begin{equation}
        \widetilde \rmP \triangleq \diag(\widetilde \rmp), \quad \widetilde \rmA \triangleq \diag(\widetilde \rma).
\end{equation}
 Using the Schur complement of $\rmM$ \cite[Section 1.4]{sayed_2022}
\begin{equation}
S \triangleq \rmM_{22} - \rmM_{21}(\rmM_{11})^{-1}\rmM_{12},
\end{equation}
which is scalar,
we can write the last row of $\rmM^{-1}$ as
\begin{align}
&[-S^{-1}\rmM_{21}(\rmM_{11})^{-1}\quad |\quad S^{-1}] \notag \\
&= S^{-1}[-\rmM_{21}(\rmM_{11})^{-1}\quad |\quad 1].
\end{align}
Note that we are only interested in finding the last row of $M^{-1}$ as only this row contributes to the FC's LBR in steady state, which is the last entry of $\widetilde{\Lambda}_{\infty}(\theta)$.

First, we find $\rmM_{11}^{-1}$. Since $\rmM_{11}$ is the sum of a diagonal matrix and a rank-one matrix, we can calculate $\rmM_{11}^{-1}$ by the matrix inversion formula \cite[Section 1.4]{sayed_2022}:
\begin{multline}
\rmM_{11}^{-1} = ({\widetilde{\rmP}} - (\rmI-{\widetilde{\rmP}}) \widetilde{\rmA})^{-1}\\
+ \frac {({\widetilde{\rmP}} - (\rmI-{\widetilde{\rmP}}) \widetilde{\rmA})^{-1}\widetilde{\rmp} {\widetilde{\rma}}^{\T}({\widetilde{\rmP}} - (\rmI-{\widetilde{\rmP}}) \widetilde{\rmA})^{-1}}{1-{\widetilde{\rma}}^{\T}({\widetilde{\rmP}} - (\rmI-{\widetilde{\rmP}}) \widetilde{\rmA})^{-1}\widetilde{\rmp}}.
\end{multline}
Consequently, 
\begin{multline}
-\rmM_{21}(\rmM_{11})^{-1} = {\widetilde{\rma}}^{\T}({\widetilde{\rmP}} - (\rmI-{\widetilde{\rmP}}) \widetilde{\rmA})^{-1}\\
+ \frac {{\widetilde{\rma}}^{\T}({\widetilde{\rmP}} - (\rmI-{\widetilde{\rmP}}) \widetilde{\rmA})^{-1}\widetilde{\rmp} {\widetilde{\rma}}^{\T}({\widetilde{\rmP}} - (\rmI-{\widetilde{\rmP}}) \widetilde{\rmA})^{-1}}{1-{\widetilde{\rma}}^{\T}({\widetilde{\rmP}} - (\rmI-{\widetilde{\rmP}}) \widetilde{\rmA})^{-1}\widetilde{\rmp}}.
\end{multline}
Observe that for $k > 1$, the $k$th element of $-\rmM_{21}(\rmM_{11})^{-1}$ is given by
\begin{align}
\frac{\pi_k}{1-\bar p_k\pi_k}\frac 1 {1-\sum_{\ell \neq 1}\dfrac{\pi_\ell p_\ell}{1-\bar p_\ell \pi_\ell}}.
\end{align}
If we define $s_k \triangleq p_k\sum_{\ell\neq k}\pi_{\ell}\bar p_\ell$, we get
\begin{align}
S &= \rmM_{22} - \rmM_{21}(\rmM_{11})^{-1}\rmM_{12} \notag \\
&= \!\!\sum_{k} \pi_kp_k \!- \!\! \sum_{k\neq 1}\frac{\pi_k s_k}{1-\bar p_k\pi_k}\frac 1 {1-\sum_{\ell \neq 1}\frac{\pi_\ell p_\ell}{1-\bar p_\ell \pi_\ell}}.
\end{align}
Now, we calculate $\widetilde{\rmU} \widetilde{\rmd}(\theta)$. Observe that
\begin{equation}
\widetilde{\rmU} \widetilde{\rmd}(\theta) = \bigg[u_2,\dots,u_K \quad | \quad (\pi_1p_1c + \sum_{k\neq 1} \pi_kp_kd_k(\theta)) \bigg]^{\T}
\end{equation}
where
\begin{equation}
    u_k = p_k\bigg(\pi_1p_1c + \sum_{\ell\neq 1} \pi_\ell p_\ell d_\ell(\theta)\bigg) + (\bar p_k\pi_kp_k + \bar p_k)d_k(\theta).
\end{equation}
As a result, it holds that
 \begin{align}\label{eq:proof_c2_big_term}
     &\widetilde{\lambda}_{\infty} (\theta) = S^{-1}[-\rmM_{21}(\rmM_{11})^{-1}\quad |\quad 1](\widetilde{\rmU} \widetilde{\rmd}(\theta)) \notag\\
     &=S^{-1}\bigg(\sum_{k \neq 1}\dfrac{\pi_kd_k(\theta)(\bar p_k\pi_kp_k + \bar p_k)}{1-\bar p_k\pi_k}\frac 1 {1-\sum_{\ell \neq 1}\dfrac{\pi_\ell p_\ell}{1-\bar p_\ell \pi_\ell}} \notag \\
     &\qquad \quad+\sum_{k\neq 1}\frac{\pi_kp_k}{1-\bar p_k\pi_k}\frac {\pi_1p_1c + \sum_{\ell\neq 1} \pi_\ell p_\ell d_\ell(\theta)} {1-\sum_{\ell \neq 1}\dfrac{\pi_\ell p_\ell}{1-\bar p_\ell \pi_\ell}} \notag \\
     &\qquad \quad+\pi_1p_1c + \sum_{\ell\neq 1} \pi_\ell p_\ell d_\ell(\theta)\bigg) \notag \\
     &=\frac{\sum_{k \neq 1}\dfrac{\pi_kd_k(\theta)(\bar p_k\pi_kp_k + \bar p_k)}{1-\bar p_k\pi_k} + \pi_1p_1c + \sum_{\ell\neq 1} \pi_\ell p_\ell d_\ell(\theta)}{S\bigg(1-\sum_{\ell \neq 1}\dfrac{\pi_\ell p_\ell}{1-\bar p_\ell \pi_\ell}\bigg)}
     \notag \\
     &=\dfrac{\sum_{k \neq 1}\dfrac{\pi_kd_k(\theta)(\bar p_k\pi_kp_k + \bar p_k)}{1-\bar p_k\pi_k} + \pi_1p_1c + \sum_{\ell\neq 1} \pi_\ell p_\ell d_\ell(\theta)}{\bigg(\sum_k \pi_kp_k\bigg)\bigg(1-\sum_{k \neq 1}\dfrac{\pi_k p_k}{1-\bar p_k \pi_k}\bigg) - \sum_{k\neq 1}\dfrac{\pi_ks_k}{1-\bar p_k\pi_k}}.
 \end{align}
Observe that the term in the numerator is equivalent to
\begin{align}
  &\sum\limits_{k \neq 1}\pi_k d_k(\theta) \Big (\dfrac{\bar p_k\pi_kp_k + \bar p_k}{1-\bar p_k\pi_k} + p_k \Big ) + \pi_1 p_1 c  \notag \\
   & = \sum\limits_{k \neq 1}\dfrac{\pi_k d_k(\theta)}{1-\bar p_k\pi_k} + \pi_1 p_1 c .
\end{align}
Furthermore, if we incorporate the following relation for $s_k$
\begin{equation}
    s_k = p_k\sum_{\ell\neq k}\pi_{\ell}\bar p_\ell = p_k\bigg(\sum_{\ell = 1}^K \pi_{\ell}\bar p_\ell - \pi_{k}\bar p_k\bigg)
\end{equation}
to the term in the denominator of \eqref{eq:proof_c2_big_term}, we obtain
\begin{align}
    &\bigg(\sum\limits_{k=1}^K \pi_kp_k\bigg)\bigg(1-\sum\limits_{k \neq 1}\dfrac{\pi_k p_k}{1-\bar p_k \pi_k}\bigg) - \sum\limits_{k\neq 1}\dfrac{\pi_ks_k}{1-\bar p_k\pi_k} \notag \\
    & = \bigg(\sum\limits_{k=1}^K \pi_kp_k\bigg)\bigg(1-\sum\limits_{k \neq 1}\dfrac{\pi_k p_k}{1-\bar p_k \pi_k}\bigg) \notag \\ & \qquad \quad - \Big (\sum_{k = 1}^K \pi_{k}\bar p_k \Big) \sum\limits_{k\neq 1}\dfrac{\pi_kp_k}{1-\bar p_k\pi_k} + \sum\limits_{k\neq 1}\dfrac{\pi_kp_k}{1-\bar p_k\pi_k} \pi_{k}\bar p_k \notag \\
    & = \bigg(1-\sum\limits_{k=1}^K \pi_k \bar p_k\bigg)\bigg(1-\sum\limits_{k \neq 1}\dfrac{\pi_k p_k}{1-\bar p_k \pi_k}\bigg) \notag \\ & \qquad \quad - \Big (\sum_{k = 1}^K \pi_{k}\bar p_k \Big) \sum\limits_{k\neq 1}\dfrac{\pi_kp_k}{1-\bar p_k\pi_k} + \sum\limits_{k\neq 1}\dfrac{\pi_kp_k}{1-\bar p_k\pi_k} \pi_{k}\bar p_k \notag \\
    & = \sum\limits_{k=1}^K \pi_k p_k - \sum\limits_{k \neq 1} \pi_k p_k \notag \\
    & = \pi_1 p_1 .
\end{align}
Thus, the expected LBR in steady state becomes
\begin{align}
    \widetilde{\lambda}_{\infty} (\theta) = \dfrac{1}{\pi_1 p_1}\sum\limits_{k \neq 1} \dfrac{\pi_k d_k (\theta)}{1-\pi_k (1-p_k)}+ c\notag\\
     = \dfrac{1}{\pi_1 p_1}\sum\limits_{k \neq 1} \dfrac{\pi_k d_k (\theta)}{1-\pi_k (1-p_k)}+ \log \frac{\mu_1(\theta^\circ)}{\mu_1(\theta)}.\label{eq:final_result}
\end{align}
Since the choice of $m = 1$ was without loss of generality, replacing the subscripts $1$ by $m$ in \eqref{eq:final_result}, we arrive at the result.

\bibliographystyle{IEEEtran}
\bibliography{ref.bib} 

\end{document}